\documentclass[]{article}

\pdfoutput=1 

\AtBeginDocument{%
  }

\usepackage[american]{babel}

\usepackage{mathtools} 
\usepackage{booktabs} 
\usepackage{tikz} 
\usepackage{comment}
\usepackage{amsfonts} 
\usepackage{amsthm}
\usepackage{float}

\usepackage{hyperref}
\usepackage{cleveref}
\usepackage{xcolor}
\usepackage{siunitx} 
\newif\ifcomments
\commentstrue
\ifcomments
    \newcommand{\anna}[1]{{\color{teal} Anna: #1}}
    \newcommand{\dan}[1]{{\color{blue} Dan: #1}}
    \newcommand{\suraj}[1]{{\color{cyan} Suraj: #1}}
    \newcommand{\hima}[1]{{\color{red} Hima: #1}}
    
\else
    \providecommand{\anna}[2][]{}
    \providecommand{\dan}[2][]{}
    \providecommand{\suraj}[2][]{}
    \providecommand{\hima}[2][]{}
\fi
\usepackage{multicol}
\usepackage{multirow}

\usepackage[most]{tcolorbox}
\usepackage{adjustbox}
\usepackage{tabularx}

\usepackage{subfig}

\usetikzlibrary{matrix}
\usetikzlibrary{positioning}
\usepackage{pgfplots}
\usepgfplotslibrary{groupplots}
\usetikzlibrary{
    pgfplots.colorbrewer,
}
\pgfplotsset{
    cycle list/.define={my marks}{
        every mark/.append style={dotted,fill=\pgfkeysvalueof{/pgfplots/mark list fill}},mark=none\\
        every mark/.append style={solid,fill=\pgfkeysvalueof{/pgfplots/mark list fill}},mark=square*\\
        every mark/.append style={solid,fill=\pgfkeysvalueof{/pgfplots/mark list fill}},mark=none\\
        every mark/.append style={solid,fill=\pgfkeysvalueof{/pgfplots/mark list fill}},mark=diamond*\\
    },
}
\usetikzlibrary{external}
\usepackage{environ}
\NewEnviron{acknowledgements}{\subsubsection*{Acknowledgements}\BODY}

\title{On Minimizing the Impact of Dataset Shifts on Actionable Explanations}

\usepackage{authblk} 
\usepackage{microtype} 
\usepackage{fullpage} 

\author[]{Anna~P.~Meyer\thanks{University of Wisconsin - Madison. Email: \texttt{apmeyer4@wisc.edu}}\hspace{0.35em}$^{\ddagger}$}
\author[]{Dan~Ley\thanks{Harvard University. Emails: \texttt{\{dley, ssrinivas\}@g.harvard.edu, hlakkaraju@hbs.edu}}\hspace{0.35em}\thanks{Equal contribution}\hspace{0.35em}}
\author[]{Suraj~Srinivas$^{\dag}$}
\author[]{Himabindu~Lakkaraju$^{\dag}$}
\affil[]{}

\begin{document}
\maketitle

\begin{abstract}
The Right to Explanation is an important regulatory principle that 
allows individuals to
request actionable explanations for algorithmic decisions.
However, several technical challenges arise when providing such actionable explanations in practice. 
For instance, models are periodically retrained to handle dataset shifts. 
This process may invalidate some of the previously prescribed explanations, thus rendering them unactionable. But, it is unclear if and when such invalidations occur, and what factors determine explanation stability i.e., if an explanation remains unchanged amidst model retraining due to dataset shifts. In this paper, we address the aforementioned gaps and
provide one of the first theoretical and empirical characterizations of the factors influencing explanation stability.
To this end, we conduct rigorous theoretical analysis to demonstrate that model curvature, weight decay parameters while training, and the magnitude of the dataset shift are key factors that determine the extent of explanation (in)stability. Extensive experimentation with real-world datasets not only validates our theoretical results, but also demonstrates that the aforementioned factors dramatically impact the stability of explanations produced by various state-of-the-art methods. 
%
%
\end{abstract}

\section{Introduction}
\label{sec:intro}

Machine learning (ML) models have recently witnessed increased utility in critical real-world applications. This, in turn, led to the introduction of several regulatory principles that aim to safeguard the practice of algorithmic decision making in such settings
~\cite{gdpr,aibillofrights}. The Right to Explanation is one such important regulatory principle that
allows individuals to request actionable explanations for algorithmic decisions that adversely impact them. Post-hoc explanation methods such as LIME, SHAP, input gradients, and Smoothgrad~\cite{shap,lime,Simonyan2013DeepIC,smoothgrad} have commonly been employed in practice to operationalize this principle. These methods explain complex model predictions by assigning importance scores to input features, typically via learning local linear approximations of the underlying functions~\cite{han2022which}.

Providing actionable explanations to end-users is often hindered by operational challenges in practice. For instance, models are periodically retrained to handle dataset shifts, and the original explanations may no longer be valid under the new model. If so, the original explanations do not remain actionable.
For example, a user may be informed that their low salary was the primary reason for the rejection of their loan application, thus prompting them to increase their income. 
However, an update to the model could result in a shift of model internals such that the user's credit score, and not their income, is now the primary factor for the rejection. In this case, the user's action of increasing their income would be less likely to result in a positive outcome once the updated model was deployed.

The aforementioned scenario could have been avoided if the underlying model and its explanations remained (relatively) unchanged after model retraining due to dataset shifts. Such \emph{explanation stability} implies actionability because if explanations remain (relatively) unchanged despite model retraining, they are likely to be actionable over an extended period of time. 
Therefore, it is necessary to develop a systematic understanding of explanation stability in the face of dataset shifts, characterize the conditions that lead to unstable explanations, and find ways to mitigate explanation instability after retraining. Despite its significance, there is little to no research on characterizing the factors that influence explanation stability. 

Our work addresses the aforementioned critical gaps by providing one of the first theoretical and empirical characterizations of the factors influencing explanation stability after dataset shifts. To this end, we conduct rigorous theoretical analysis to demonstrate that model curvature, weight decay parameters while training, and the magnitude
of the dataset shift are key factors that determine the extent of explanation (in)stability.
Our theoretical analysis emphasizes how seemingly inconsequential modeling decisions that may not impact predictive accuracy can heavily influence other critical aspects such as explanation stability. We also conduct extensive experimentation with real-world datasets to validate our theoretical insights with various state-of-the-art explanation methods (e.g., gradient and perturbation based methods such as LIME, SHAP, SmoothGrad etc.) and explanation stability metrics (e.g., $\ell_2$, top-k consistency).  
We also empirically analyze how other training decisions (e.g., learning rate and number of training epochs) impact explanation stability. 
In summary, our work makes the following key contributions:

\begin{enumerate}
    \item We provide one of the first theoretical and empirical characterizations of the factors influencing explanation stability in the face of dataset shifts.
    \item We conduct a rigorous theoretical analysis to demonstrate that model curvature, weight decay parameters during training, and the magnitude of the dataset shift are key factors influencing
the extent of explanation (in)stability (\S \ref{sec:theory}).
    \item We carry out extensive experimentation with multiple real-world datasets to validate our theoretical results (\S \ref{sec:finetune}). Our empirical findings suggest that standard neural network training pipelines exhibit low explanation stability, and confirm that our theoretical results are valid even when some of the underlying assumptions do not hold in practice (\S \ref{sec:extendingtheory}). 
    \item We empirically analyze the impact of other training hyperparameters such as learning rate, batch size, and number of training epochs on explanation stability (\S \ref{sec:sensitivity}).
\end{enumerate}

\section{Related Work}\label{sec:related}

\paragraph{Explanation Methods} 
A variety of post-hoc techniques have been proposed to explain complex models~\cite{DoshiVelez2017TowardsAR,koh2017understanding,ribeiro2018anchors}. 
These techniques differ in their access to the complex model (i.e., black box vs. access to internals), scope of approximation (e.g., global vs. local), search technique (e.g., perturbation-based vs. gradient-based), explanation families (e.g., linear vs. non-linear), etc. 
For instance, LIME~\cite{lime} and SHAP~\cite{shap} are \emph{perturbation-based}, \emph{local explanation} approaches that learn a linear model locally around each prediction. 
Other \emph{local explanation} methods capture feature importances by computing the gradient with respect to the input~\cite{selvaraju2017grad, Simonyan2013DeepIC, smoothgrad,sundararajan2017axiomatic}. 
Counterfactual explanation methods, on the other hand, capture the changes that need to be made to a given instance in order to flip its prediction~\cite{karimi2020algorithmic, karimi2020causal, looveren2019interpretable, face, ustun2019actionable,wachter2017counterfactual, mace}. 
In this work, we focus on analyzing the stability of explanations output by perturbation-based and gradient-based local explanation methods. 

\paragraph{Explanation Stability and Robustness}
Various notions of explanation stability and robustness have been suggested in the literature.\footnote{The literature is inconsistent on the distinction between explanation robustness and stability; for clarity we use ``robustness'' when the model is fixed (e.g., robustness to input perturbations) and ``stability'' when the model is changing.} One line of work shows that explanations are not robust by crafting adversarial manipulations of test inputs that change their explanations but not predictions~\cite{dombrowski2019explanations, ghorbani}. Another set of approaches studies explanation stability by altering the underlying model to maintain accuracy, but change explanations to a desired target~\cite{andersFairwashing, heo2019fooling, slack2020fooling}. Multiple works suggest using smooth or low-curvature models to improve both model and explanation stability~\cite{dombrowski2019explanations, srinivas2022efficient}.
However, all of these works focus on adversarial manipulations of either the model or the input, while we are interested in more realistic model shifts occurring due to natural dataset shifts. In contrast, other works have proposed algorithmic techniques to generate counterfactual explanations or algorithmic recourses that are stable under model shifts~\cite{bui2022counterfactual, dutta2022robust, forel2022robust, nguyen2022distributionally, nguyen2022robust, rawalRecourse, upadhyay2021}. However, these works differ from ours in three ways: first, we focus on feature attribution-based explanations that highlight influential features, while the existing work focuses on counterfactual explanations or algorithmic recourse. Second, we adapt the model to yield more stable explanations after retraining due to dataset shifts, whereas most of the existing work  aims to adapt the explanation method. And finally, we theoretically characterize the factors that impact explanation stability.

More generally, neural networks are known to be non-robust to small training modifications, e.g., different random initializations or alternate model selection due to underspecification~\cite{black-leaveoneout, damour2022underspecification, kolenBack, mehrerIndividual}.
However, relatively few works have studied how \emph{explanations} change in the same setting. These works only show empirically that explanations are not stable with respect to underspecification~\cite{brunet2022implications} or random model initializations~\cite{black2022selective}.  These works are orthogonal to ours, as we focus on theoretically and empirically characterizing specific properties of models and datasets that impact explanation stability.

\paragraph{Algorithmic Stability}
Algorithmic stability is a classical learning-theoretic framework that characterizes the consistency of outputs of learning algorithms when trained on different but similarly distributed data \cite{bousquet2002stability}. In contrast, our study does not make the assumption that the shifted data is similarly distributed. 
\cite{hardt2016train} studied the hypothesis stability of stochastic gradient descent (SGD) and showed that practical choices such as learning rate and number of epochs of training can impact the stability of model parameters. 
Recent work has applied the algorithmic stability perspective to post-hoc explanations~\cite{Fel_2022_WACV} by requiring that feature attributions be stable to a re-sampling of the underlying dataset. However, none of these prior works study realistic model shifts occurring due to natural dataset shifts or characterize the factors influencing explanation stability, which is the main focus of our work.

\newcommand{\X}{\mathbf{x}}
\newcommand{\R}{\mathbb{R}}
\newcommand{\E}{\mathbb{E}}
\newcommand{\D}{\mathcal{D}}
\newcommand{\grad}{\nabla_{\X}}
\newcommand{\p}{\|}

\newtheorem{definition}{Definition}
\newtheorem{lemma}{Lemma}
\newtheorem{theorem}{Theorem}
\newtheorem{assumptions}{Assumptions}
\Crefname{definition}{Defn.}{Defns.}

\newenvironment{hproof}{%
  \renewcommand{\proofname}{Proof Idea}\proof}{\endproof}

\section{Theoretical Analysis}\label{sec:theory}

In this section, we provide an analytical characterization of explanation shifts $e_1 \rightarrow e_2$ and model shifts $f_1 \rightarrow f_2$ that happen due to underlying dataset shifts $\mathcal{D}_1 \rightarrow \mathcal{D}_2$. Concretely, we address the following question: 
\noindent\fbox{%
    \parbox{0.96\textwidth}{%
How much do the explanations for a model $f_1$ trained on dataset $\D_1$ change when fine-tuning\footnotemark on a slightly shifted dataset $\D_2$, resulting in a new model $f_2$?
}}\footnotetext{
Here, fine-tuning refers to initializing model parameters of $f_2$ with those of $f_1$, and then training $f_2$ using $D_2$ until its loss converges. We provide practical details of this in \S \ref{sec:finetune}.}

Let the dataset $\D_1$ be composed of input-output pairs ($\X, y$), where $\X \in \R^d$, and $y$ is either a real value or a one-hot vector depending on the task (regression or classification). For simplicity of analysis, in this section we shall focus on gradient explanations $e = \grad f \in \R^d$, where $f$ is a scalar-valued neural network function $f(\cdot, \theta):\R^d \rightarrow \R$ mapping $D$-dimensional inputs to scalar outputs.
Let $\ell(\cdot, \theta): \R^d \rightarrow \R$ denote $f$ composed with a loss function such as cross-entropy loss. 
Here, $\theta$ represents the parameters of the function being minimized, e.g., a deep neural network.

Our goal here is to relate the explanation shift to the dataset shift (i.e., $d(\D_1, \D_2)$, for some distance measure $d$) and the learning algorithm. To this end, our overall strategy is: 
\begin{enumerate}
    \item In \S \ref{subsec:param_shift}, we bound the \emph{parameter shift} $\| \theta_2 - \theta_1 \|_2$ in terms of the \emph{dataset shift} $d(\D_1, \D_2)$.
    \item In \S \ref{subsec:expln_shift}, we bound the \emph{explanation shift} $\| \grad f(\X, \theta_2) - \grad f(\X, \theta_1) \|_2 $ in terms of the \emph{parameter shift} $\| \theta_2 - \theta_1 \|_2$ derived above.
\end{enumerate}

While the existence of these shifts are qualitatively evident, 
we make novel contributions in that (a) we quantify the parameter and gradient changes and (b) we show that these changes are affected by factors in the modelling process. 

\subsection{Bounding the Parameter Shift}\label{subsec:param_shift}

We begin by defining the distance between datasets $d(\D_1, \D_2)$ using Hungarian distance, as defined in \Cref{defn:hungarian}. The idea behind the Hungarian distance is simple: we find the "alignment" between data points of two datasets that minimizes the average $\ell_2$ distance between a point in $\D_1$ and its counterpart in $\D_2$. 

\begin{definition}\label{defn:hungarian}
Given datasets 
$\D_1=\{(x_i, y_i)_{i=1,\ldots,N}\}\ $ 
and 
$\D_2 = \{ (x'_i, y'_i)_{i=1,\ldots,N}\}$,  
their \emph{Hungarian distance} is
    \begin{align*}
       d(\D_1, \D_2) = \min_{P(\D_2)} \sum_{i=1}^N \| x_i - x'_i \|_2
    \end{align*}
where the minimum is taken over $P(\D_2)$, all permutations on the data points of $\D_2$.
\end{definition}

Using this distance metric, we derive an expression for the parameter change $\| \theta_2 - \theta_1 \|_2$ in terms of the dataset distance $d(\D_1, \D_2)$. We use the following assumptions to derive this result: (1) we minimize a regularized loss $\ell_{reg}$ that involves weight decay, i.e, $\ell_{reg}(\theta) = \ell(\theta) + \gamma \| \theta \|^2_2$; (2) $\ell$ is locally quadratic; (3) the learning algorithm returns a unique minimum $\theta$ given a dataset $\D$.

\begin{theorem} Given the assumptions stated above, and that $\mathcal{L}_x (\theta_1)$ is the Lipschitz constant of the model with parameters $\theta_1$, we have 

\begin{align*}
    \p \theta_2 - \theta_1 \p_2 & \leq  \sqrt{\frac{ \mathcal{L}_x(\theta_1) d(\D_1, \D_2)}{\gamma }} + C
\end{align*}

where $\gamma$ is the weight decay regularization constant, and $C$ is a small problem-dependent constant.

\end{theorem}

\begin{hproof}
The idea of the proof is to first estimate $\ell_{\D_2}(\theta_1)$, i.e., the initial loss of the model $\theta_1$, before fine-tuning on $\D_2$. Assuming, for the sake of explanation, that the optimal loss value obtained by $\theta_2$ on $\D_2$ is zero, we have obtained the change in loss value from $\theta_1 \rightarrow \theta_2$. To derive the change in parameters from the change in loss derived above, we use a second order Taylor series expansion (with the assumption that the loss is locally quadratic). In particular, we lower bound the Hessian with its lowest eigenvalue using the weight decay term $\gamma$.
The ``problem-dependent constant'' $C$ arises from deviating from the assumption that the optimal loss value is zero. The complete proof is provided in the supplementary material.
\end{hproof}

The above theorem achieves our first goal of relating parameter change to dataset shift. Intuitively, this tells us that parameter shift depends directly on dataset shift (as expected), directly on the Lipschitz constant (which depends on the model's robustness), and inversely on the weight decay parameter.

\subsection{Bounding the Explanation Shift}\label{subsec:expln_shift}
 
We now consider the problem of bounding the explanation shift $\| \grad f(\cdot, \theta_2) - \grad f(\cdot, \theta_1) \|_2$ given an estimate of the parameter shift $\| \theta_2 - \theta_1 \|_2$. To this end, we first define a quantity called the \textit{gradient-parameter Lipschitz constant}, defined in \Cref{defn:pilp}. To the best of our knowledge, this quantity has not been considered in previous works. 

\begin{definition}\label{defn:pilp} We define the ``gradient-parameter Lipschitz constant'' $\mathcal{L}_{\Theta, \D}$ w.r.t. an input distribution $\D$, and parameter set $\Theta$ as follows:

\begin{align*}
    \mathcal{L}_{\Theta, \D} = \E_{\theta \in \Theta} \E_{\X \sim \D} \p \nabla_\theta \nabla_x f(x; \theta) \p_2
\end{align*}

\end{definition}

Intuitively, this quantity captures the sensitivity of gradient explanations to small changes in the parameter $\theta$. In contrast to usual definitions of Lipschitz constants, this is defined (1) locally with respect a particular distribution and parameter set, and (2) using the mean instead of the $\sup$. This allows us to derive the following relationship between average gradient difference and parameter shift.

\begin{lemma}
The gradient-parameter Lipschitz has the following property: 

\begin{align*}
    \E_{\X \sim \D} \p \nabla_x f(x; \theta_1) - \nabla_x f(x; \theta_2) \p_2 \leq~ \mathcal{L}_{\Theta, \D} \times  \p \theta_2 - \theta_1 \p_2
\end{align*}

where $\Theta = \{ \lambda \theta_1 + (1 - \lambda) \theta_2 \mid \lambda \in [0,1] \}$.

\end{lemma}

This relation follows immediately from the fundamental theorem of integral calculus, and the full proof is given in the supplementary material. This result reveals an intuitive fact that larger parameter shifts lead to larger gradient shifts, but mediated by the gradient-parameter Lipschitz. This quantity is difficult to analyse for general neural networks, and we provide below an analysis for the case of 1-hidden layer neural networks, with a specific simplifying assumption on the data distribution.

\begin{theorem}
Assume that we have a 1-hidden layer neural network with weights $\theta$, and random inputs $\X \sim \mathcal{N}(0, I)$\footnote{Covariance of I is chosen for notational brevity}. Further assume that we use an activation function $\sigma$ with well-defined second derivatives (e.g: softplus). For this case, the gradient-parameter Lipschitz constant is

\begin{align*}
   \E_{\X, \theta} \| \nabla_{\theta} \grad f(\X, \theta) \|_2 \leq \left( \E_{\theta} \| \theta \|_2 \right) + \beta \left( \E_{\theta} \phi(\theta) \right)
\end{align*}

where $\beta$ is the maximum curvature of activation function $\sigma$, and $\phi(\theta)$ is the path-norm \cite{neyshabur2015norm} of the model.
\end{theorem}

\begin{hproof}
The proof involves computing the gradient-parameter derivatives for the 1-hidden layer neural network case. Our assumption regarding smoothness of activation function ensures that $\sigma'' \leq \beta$, i.e., that the second derivative of the activation function is upper bounded, which is true for smooth non-linearities like softplus. Another simplifying assumption regarding the distribution of the inputs ($\X \sim \mathcal{N}(0, I)$) helps us dramatically simplify the expression for the expected value of the gradient norm terms. 

The path-norm for a 1-hidden layer neural network (for weights $W_2, W_1$) is given by $\phi(\theta) = \sqrt{\sum_{i,j} (W_2^j W_1^{i,j})^2} $ and has been linked to model generalization \cite{neyshabur2015norm}. The complete proof is provided in the supplementary material.
\end{hproof}

These results achieve our goal of relating the change in gradients to parameter change. Taken together with the results in \S \ref{subsec:param_shift}, this achieves our overall goal of relating gradient change to dataset shift. Intuitively, this tells us the following regarding gradient shift:
\begin{enumerate}
    \item It depends directly on the \emph{dataset distance} $d(\D_2, \D_1)$, which is as expected: the larger the dataset shift, the larger can be the parameter shift, and explanation shift.
    \item It depends inversely on the \emph{weight decay} parameter $\gamma$, and directly on the norm of weights $\| \theta \|_2$. Intuitively, large weight decay shifts the optima to be closer to zero, reducing the norm of weights, which makes all optima closer to each other. 
    \item It depends directly on the \emph{Lipschitz constant} of the first model, which is connected to the model's robustness. Intuitively, if the model is already robust to small input changes, we don't expect it to shift too much upon fine-tuning on slightly shifted data.
    \item It depends directly on the \emph{smoothness constant} $\beta$ of the activation function used in the model. The smoother the model, the smaller are its second derivatives, and thus its gradient-input Lipschitz.  
\end{enumerate}

These insights motivate the core thesis of the paper, i.e., that explanation stability depends on properties beyond predictive performance. Specifically, we see that specific algorithm choices that do not necessarily improve accuracy, such as weight decay, smoothness constant and robustness-inducing losses, affect explanation and model stability to data shifts.

We conclude by making two remarks regarding the theory. First, it only highlights \emph{sufficient conditions} for explanation stability, not necessary conditions. Meaning, there can be ways to achieve gradient stability without the parameters $\theta_1, \theta_2$ being close, or without neural nets being Lipschitz, or when ReLU or weight decay are not used, but these lie outside the scope of this theory. Second, the focus of these bounds is not to produce numerically tight estimates for explanation stability, rather, it is intended as a tool to make \textit{qualitative predictions} regarding specific modelling interventions that can improve explanation stability in practice. \Cref{sec:eval} explores this in more detail.





\section{Experimental Evaluation}\label{sec:eval}

In this section, we describe our general evaluation setup. In \S \ref{sec:finetune}, we experimentally validate that the modeling interventions identified in \S \ref{sec:theory} improve explanation stability. Next, in \S \ref{sec:extendingtheory}, we show that these relationships extend to realistic training setups and more-sophisticated explanation techniques (i.e., top-K feature attributions of techniques like SmoothGrad, LIME, and SHAP). Finally, in \S \ref{sec:sensitivity}, we show that other hyperparameters not discussed in our theory -- e.g., learning rate, learning rate decay, batch size, and number of training epochs -- also impact gradient stability. 

\paragraph{Datasets} We evaluate explanation stability on three binary datasets with continuous features. We use the WHO life expectancy dataset ($n$=2928, number of features $d$=18)~\cite{whoDataset}, which evaluates whether a country's life expectancy is above the global median based on health and economic factors; the HELOC dataset ($n$=9871, $d$=23)~\cite{heloc}, which evaluates loan acceptance given the applicant's financial information; and the Adult Income dataset ($n$=32561, $d$=6)~\cite{uci-data}, which evaluates whether a person's income is above \$50,000. We model dataset shifts in two ways:

\begin{itemize}
    \item \textit{Synthetic Noise.} We add Gaussian $\mathcal{N}(0,\sigma^2)$ noise to all training samples, with varying levels of noise $\sigma$.
    \item \textit{Temporal Shift.} We create a temporal shift dataset from the WHO data by using pre-2012 data as the ``original'' dataset and the full dataset as the ``shifted'' dataset.
\end{itemize}

\paragraph{Models} We evaluate the stability of medium-sized neural networks containing 5 layers of 50 nodes. By default, we use ReLU activation; however, to see the effects of model curvature on explanation stability, we use softplus (SP) for some experiments. For the experiments using SP, we vary the parameter $\beta$ from 2 to 10 depending on the dataset and the experiment's goals. Smaller values of $\beta$ correspond to lower-curvature models, while SP with $\beta=\infty$ is equivalent to ReLU. Unless otherwise specified, all models we compare within each section achieve similar predictive accuracy. 

\paragraph{Experimental procedure} For each experiment, we average over 10 trials to reduce the impact that random seeds  have on the outcomes. Each trial in the synthetic noise experiments consists of training one ``base'' model plus 10 ``noisy'' models, each corresponding to a different random perturbation of the data. In the temporal shift experiments, each trial compares the base model (trained on the partial dataset) to a new model (trained on the full dataset). We compute explanation stability for all test samples, except for LIME and SHAP results, where we limit computations to 100 randomly-chosen test samples. In the results, we display the mean of all trials along with the middle 50\% of values.

\subsection{Validating Our Theoretical Results}\label{sec:finetune}
\paragraph{Experiment setup}
In this section, we test our theoretical results from \S \ref{sec:theory} by evaluating the impact of \emph{weight decay}, \emph{model curvature}, and \emph{size of the data shift} on explanation stability. To ensure that our analysis is sound, we adhere to the theoretical assumptions as closely as possible. In particular, we focus on \emph{model parameters} ($\theta$) and \emph{gradients} ($g$) rather than top-K feature attributions, since top-K explanations are not differentiable. We train each base model for a large number of epochs to ensure convergence of the training loss to a minimum (i.e., we ignore overfitting on the test set). To fine-tune, we copy the parameters of the base model and then train on the shifted data. Also, we start fine-tuning with a lower learning rate and train for fewer epochs, given that the initial loss is already close to a minimum.

\begin{figure}[t] 
    \centering
    \includegraphics[width=0.48\textwidth]{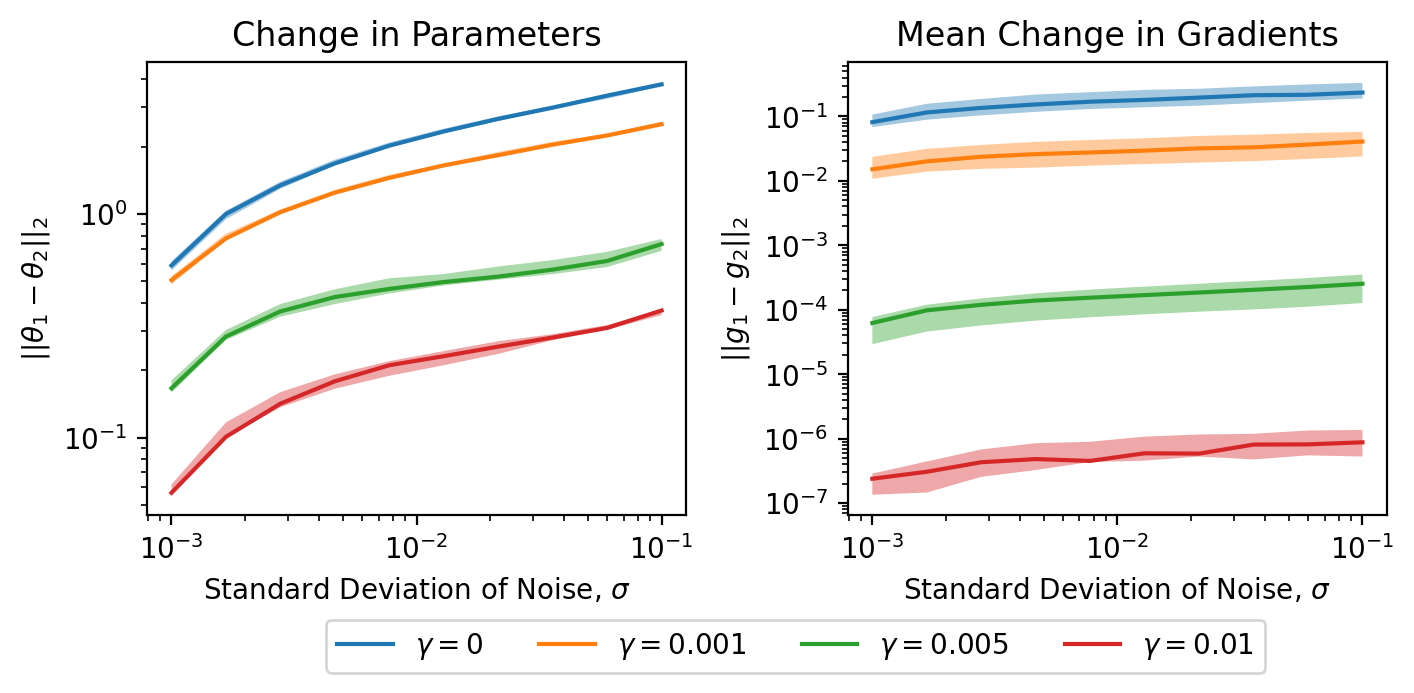}
    \caption{\small Effect of weight decay value on the mean gradient and parameter changes for HELOC after fine-tuning on a shifted dataset. All models use ReLU activation. The x-axis is the size of the data shift, represented by the standard deviation of noise, $\sigma$.}
    \label{fig:fine_tune_gamma_heloc}
\end{figure}

\begin{table}[t] 
    \small
    \centering
    \caption{\small Effect of weight decay ($\gamma$) on WHO (with fine-tuning).}
    \begin{tabular}{l|rrr}\toprule
   & $\gamma=0$ & 
    $0.001$ & $0.01$ \\\midrule
    $\|\theta_1-\theta_2\|_2$ & 1.69\textpm 0.09
            &  1.53\textpm 0.03 
            &  0.95\textpm 0.02
            \\
    $\|g_1-g_2\|_2$  & 0.201\textpm 0.065
        &0.122\textpm 0.034 
        & 0.010\textpm 0.002
        \\\bottomrule
    \end{tabular}
    \label{tab:wd_finetune}
\end{table}

For brevity, we omit similar results from the Adult dataset, with complete results in the appendix. We also provide the training hyperparameters for each dataset in the appendix.

\paragraph{Weight decay results} \Cref{fig:fine_tune_gamma_heloc} and \Cref{tab:wd_finetune} show the effect of weight decay for the synthetically and naturally-shifted datasets, respectively.  
We see that larger weight decay values correspond to orders of magnitude smaller mean gradient and parameter changes, which is in line with the theoretical relationship between weight decay and stability. 
As the amount of synthetic noise grows, we observe that both parameters and gradients diverge more between the base model and fine-tuned model. This pattern supports our hypotheses that (a) the loss landscape changes more with larger data shift and (b) a larger change in model parameters is correlated with a larger change in test sample gradients.

\paragraph{Curvature results} \Cref{fig:ft_curvature_heloc} and \Cref{tab:curv_finetune} show that using low-curvature training techniques increases gradient stability.  In particular, SP does better than ReLU, particularly with smaller values of $\beta$. For both the real-world and synthetic shifts, the gradient changes are less distinct (i.e., the confidence intervals overlap more) than the parameter changes, but still show a trend in the predicted direction. 

\begin{figure}[t] 
    \centering
    \includegraphics[width=0.48\textwidth]{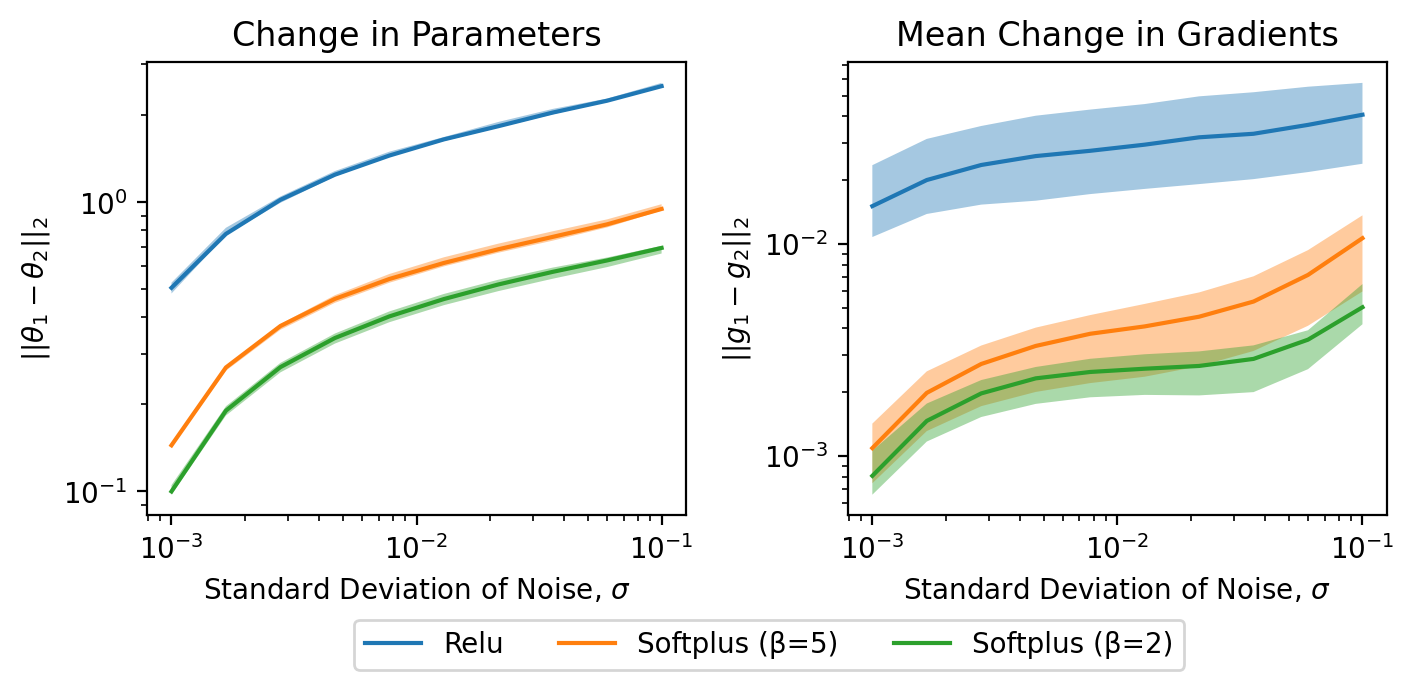}
    \caption{\small  Effect of model curvature on the mean gradient and parameter changes for the HELOC dataset after \textbf{fine-tuning} on a shifted dataset. The x-axis is the size of the data shift, represented by the standard deviation of noise, $\sigma$.}
    \label{fig:ft_curvature_heloc}
\end{figure}

\begin{table}[t] 
    \small
    \centering
    \caption{\small Effect of curvature on WHO (with \textbf{fine-tuning}).}
    \begin{tabular}{l|rrr}\toprule
   & ReLU  & SP ($\beta=10$) & SP ($\beta=5$) \\\midrule
    $\|\theta_1-\theta_2\|_2$ & 1.53\textpm 0.03
            & 1.37\textpm 0.05
            &  1.23\textpm  0.02
            \\
    $\|g_1-g_2\|_2$&0.122\textpm 0.034
            & 0.096\textpm 0.026
            & 0.070\textpm 0.024
            \\\bottomrule
    \end{tabular}
    \label{tab:curv_finetune}
\end{table}

\subsection{Going Beyond Theoretical Results}\label{sec:extendingtheory}

In this section, we evaluate whether our theoretical predictions still hold if we deviate from the theoretical assumptions, namely: fine-tuning, exactly reaching a local minimum, and using non-differentiable top-K feature attribution metrics rather than gradients $\ell_2$ distance. 
The top-K features for a model's predictions are the $k$ features with attribution values (e.g., gradients or functions of gradients for gradient-based methods; weights of linear models for LIME and SHAP) of highest magnitude. For example, if an instance $\mathbf{x}$ comprises of four features indexed by $[1 \cdots 4]$, and the gradient (feature attribution) of a given model w.r.t. $\mathbf{x}$ is $[-0.1, -0.4, 0.2, 0.3]$, then the top-2 features are features 2 and 4 since their corresponding feature attribution magnitudes -0.4 and 0.3 are the highest.

\paragraph{Experiment setup}
 We maintain our experimental setup from \S \ref{sec:finetune}, except that we train each model for a smaller number of epochs (until accuracy stabilizes), before retraining a model from scratch on the shifted data, rather than fine-tuning (hyperparameters are located in the appendix). We use the same random network initialization for both training and retraining, as prior work has shown that feature attributions are not stable to changes in random seed~\cite{black2022selective}.

\begin{figure}[t] 
    \centering
    \includegraphics[width=0.494\textwidth]{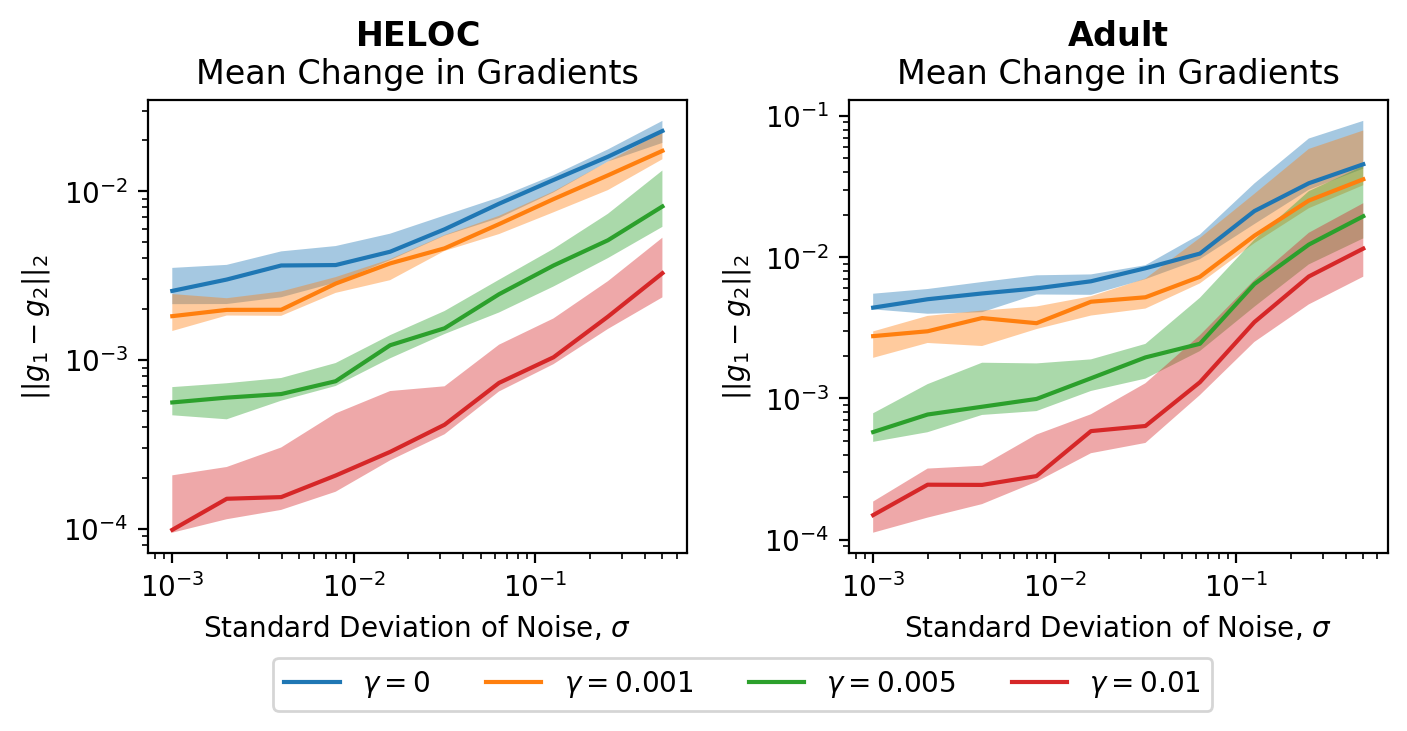}
    \caption{\small Effect of weight decay on parameter stability when \textbf{retraining} on HELOC and Adult datasets. The x-axis is the size of the data shift, represented by the standard deviation of noise, $\sigma$.}
    \label{fig:retraining_gamma_grads}
\end{figure}

\begin{table}[t] 
    \small
    \centering
    \caption{\small Effect of weight decay ($\gamma$) on WHO (with \textbf{retraining}).}
    \begin{tabular}{l|rrr}\toprule
   & $\gamma=0$ & 
    $0.001$ & $0.01$ \\\midrule
    $\|\theta_1-\theta_2\|_2$ & 3.97\textpm 0.13
            &  3.50\textpm 0.19
            &  1.79\textpm 0.17
            \\
    $\|g_1-g_2\|_2$  & 0.295\textpm 0.075
        &0.150\textpm 0.034
        & 0.003\textpm 0.001
        \\\bottomrule
    \end{tabular}
    \label{tab:wd_scratch}
\end{table}

\begin{table}[t]
    \small
    \centering
    \caption{\small Effect of curvature on WHO (with \textbf{retraining}).}
    \begin{tabular}{l|rrr}\toprule
   & ReLU  & SP ($\beta=10$) & SP ($\beta=5$) \\\midrule
    $\|\theta_1-\theta_2\|_2$ &  3.50\textpm 0.19
            &  2.60\textpm 0.05
            &   1.97\textpm  0.12
            \\
    $\|g_1-g_2\|_2$& 0.150\textpm 0.034
            &   0.102\textpm 0.040
            &  0.076\textpm 0.024
            \\\bottomrule
    \end{tabular}
    \label{tab:curv_scratch}
\end{table} 

\begin{figure}[t]
    \centering
    \includegraphics[width=0.48\textwidth]{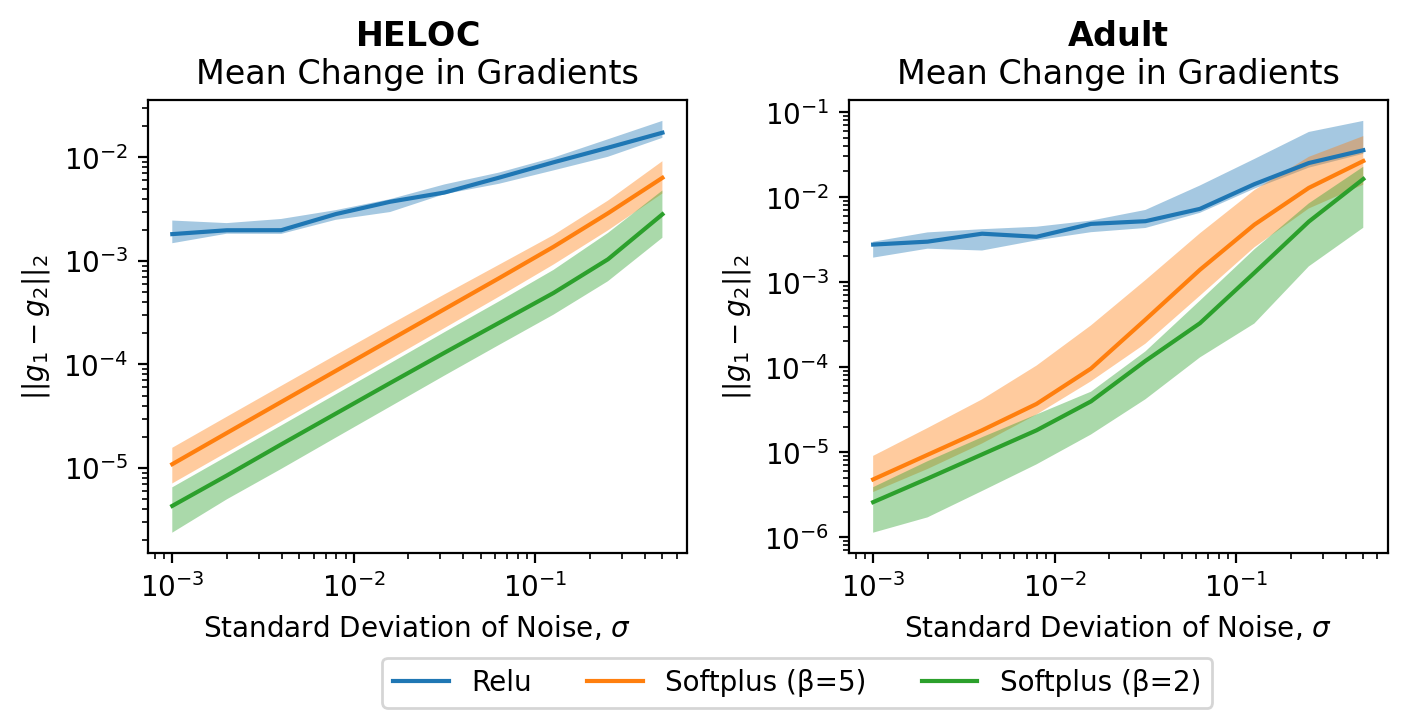}
    \caption{\small Effect of model curvature on gradient stability when \textbf{retraining} on HELOC and Adult datasets. The x-axis is the size of the data shift, represented by the standard deviation of noise, $\sigma$.}
    \label{fig:retraining_curvature_grads}
\end{figure}

\begin{figure*}[t]
    \centering
    \includegraphics[width=0.95\textwidth]{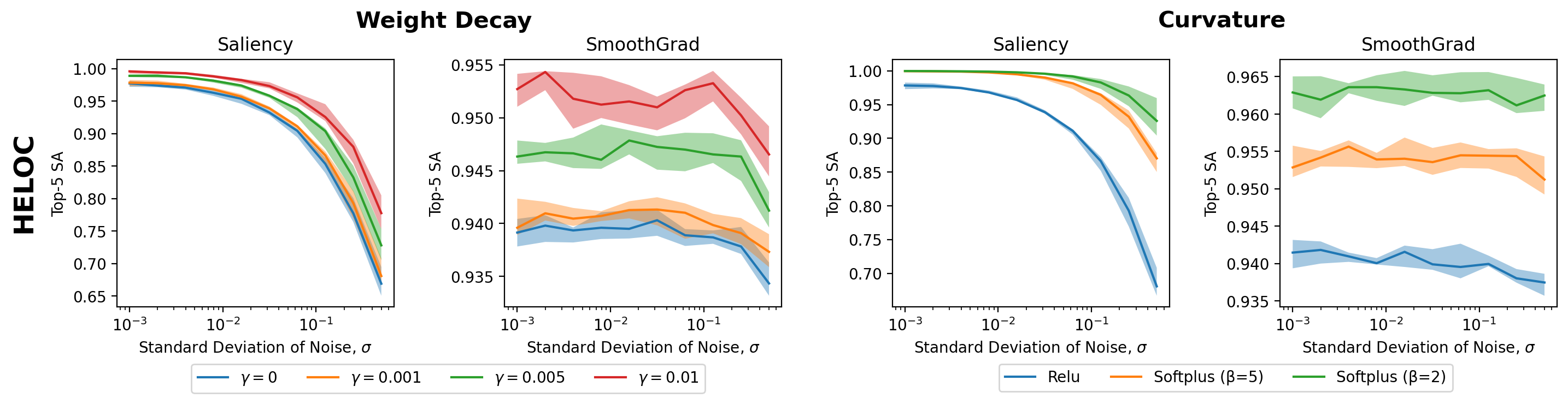}
    \caption{ \small Top-5 SA for the HELOC dataset under model retraining. The left two graphs show the effect of weight decay and the right two graphs show the effect of curvature. Graphs 1 and 3 show Top-5 SA for saliency, while graphs 2 and 4 show Top-5 SA for SmoothGrad. Confidence intervals represent the middle 50\% of values. See the appendix for other datasets and explanation techniques.}
    \label{fig:explanations_retraining}
\end{figure*}

To complete the analysis, we consider four feature attribution techniques: input gradients, input gradients with SmoothGrad, LIME, and kernel SHAP~\cite{shap,lime,smoothgrad,Simonyan2013DeepIC}. 
We use three top-K stability metrics adapted from Brunet et al.~\cite{brunet2022implications}. They each take in a pair of top-K feature sets corresponding to two models. Sign Agreement (SA) returns the fraction of top-K features that appear in both models' top-K features and have the same signed value. Consistent Direction of Contribution (CDC) is binary (per-sample) and measures whether all features in the top-K (for either model) have the same signed value in the other model. Signed-Set Agreement (SSA) is also binary (per-sample) and is 1 if the two model's top-K feature sets contain the same features and the features have the same signed value (the order of the top-K features does not matter).

\paragraph{Relaxing Training Assumptions}
From \Cref{tab:wd_scratch}, we see that increasing the algorithm's weight decay significantly decreases both parameters and gradients.  
\Cref{fig:retraining_gamma_grads} confirms that a larger weight decay during training also increases gradient stability for models with synthetic noise (analogous results for parameter stability are included in the appendix). 
\Cref{fig:retraining_curvature_grads} and \Cref{tab:curv_scratch} show that the curvature trends discovered in \S \ref{sec:finetune} still hold, i.e., lower model curvature induces higher gradient stability.

Although curvature and weight decay increase explanation stability upon retraining, we note parameter stability scores are much lower for retraining than for fine-tuning given fixed curvatures and weight decays. For example, the mean change in parameters when using ReLU with a weight decay of 0.001 is 1.53\textpm 0.03 when fine-tuning, versus 3.50\textpm 0.19 when retraining. The difference in gradient change between the two techniques is less stark, however, suggesting that even though the fine-tuned models are more similar to each other overall, the retrained models are sufficiently similar on the parts of the input domain represented in the test set.

\paragraph{Using Diverse Explanation Techniques}

\Cref{fig:explanations_retraining} shows Top-5 SA stability for two popular explanation techniques, namely, Saliency (input gradients) and SmoothGrad~\cite{kokhlikyan2020captum}. As in the previous experiments, models with low curvature and large weight decay are more effective at preserving explanation stability. In case of saliency, the explanation stability starts to decrease sharply for shift magnitudes larger than $\mathcal{N}(0,0.1^2)$. For SmoothGrad, the noise has less of an effect, which we hypothesize is because SmoothGrad is more stable overall as it computes an averaged gradient over a small local neighborhood of a given instance (see \Cref{tab:all_res}). Lowering the model's curvature and increasing the weight decay similarly improve explanation stability under data shifts for the other metrics and explanation techniques (see the appendix for a full set of graphs).

\Cref{tab:all_res} gives a more detailed breakdown of top-K feature attribution stability across all datasets and explanation techniques. The table corresponds to a single modeling choice -- ReLU with weight decay 0.001. We notice trends based on the \emph{explanation technique}, \emph{explanation metric}, and \emph{dataset}. 

The different explanation techniques exhibit vastly different stabilities under dataset shifts. Notably, SmoothGrad outperforms all other explanation techniques, obtaining scores of over 90\% for both Top-3 and Top-5 explanation stability for all of the metrics we used. That is, over 90\% of test samples' attributions had the same signs for all of each model's top-K features (CDC), over 90\% of test samples' attributions had the same signed features in the top-K (SSA), and on average, over 90\% of the top-K features for each test sample appeared in the other model's top-K features (SA).

To compare the explanation metrics, we look at the relative difference between SSA and each of SA and CDC (since SSA is strictly stronger than either of the other metrics). We see that neither SA nor CDC is consistently more stable, suggesting that the failure to have identical top-K feature sets stems both from different features ranking as the most influential (i.e., what SA captures) and also whether the top features have the same sign in the other model (i.e., CDC).

We also see that different datasets exhibit different levels of explanation stability. The top-3 and top-5 scores for the Adult datasets are very high; however, this dataset only has six features, so there are fewer possible feature combinations. We see that SSA scores for Adult are much lower than the SA scores, indicating that even with only 6 features, the learnt models often disagree on different features' impacts.

\begin{table*}[t]
\small 
\centering
\caption{\small Explanation stability scores for all datasets. The synthetically-shifted datasets (HELOC and Adult) are modified by $\mathcal{N}(0,0.1^2)$ noise. Models use ReLU activation and weight decay $\gamma=0.001$. All values averaged across all test samples and then across all trials. Error bounds are such that the lower and upper bound values represent the 25$^\text{th}$ and 75$^\text{th}$ percentile respectively.}
\label{tab:all_res}
\begin{tabular}{ll|lll|lll}\toprule
                       & Explanation     & \multicolumn{3}{c|}{Top-3} & \multicolumn{3}{c}{Top-5} \\
Dataset                & technique   &  SA          & CDC       & SSA                     & SA      & CDC    & SSA    \\\midrule
\multirow{4}{*}{WHO}   & Saliency    & 0.63\textpm 0.01   & 0.83\textpm 0.03  & 0.18\textpm 0.03  & 0.61\textpm 0.03   & 0.61\textpm 0.10  &  0.03\textpm 0.01 \\
                       & SmoothGrad  & 0.94\textpm 0.00 & 0.94\textpm 0.01 & 0.91\textpm 0.00 &   0.94\textpm 0.00 & 0.92\textpm 0.00  &  0.90\textpm 0.00 \\
                       & LIME        & 0.69\textpm 0.09 & 0.35\textpm 0.24  & 0.27\textpm 0.20 & 0.59\textpm 0.57   & 0.08\textpm 0.05  & 0.06\textpm 0.05  \\
                       & K.SHAP      & 0.58\textpm 0.10  & 0.60\textpm 0.39  & 0.13\textpm 0.08 & 0.61\textpm 0.10   & 0.41\textpm 0.31  &  0.05\textpm 0.05 \\\midrule
\multirow{4}{*}{HELOC} & Saliency    &  0.51\textpm 0.02  & 0.88\textpm 0.03  & 0.07\textpm 0.01  &  0.55\textpm 0.02  &  0.71\textpm 0.05  &  0.01\textpm 0.00  \\
                       & SmoothGrad  & 0.94\textpm 0.00 & 0.96\textpm 0.00 & 0.90\textpm 0.00 & 0.94\textpm 0.00 &  0.94\textpm 0.00 &  0.90\textpm 0.00 \\
                       & LIME        & 0.63\textpm 0.02 & 0.45\textpm 0.04  & 0.18\textpm 0.03 &  0.66\textpm 0.02  &  0.16\textpm 0.03 &  0.07\textpm 0.02 \\
                       & K.SHAP      &  0.64\textpm 0.02  & 0.91\textpm 0.03  & 0.19\textpm 0.03 &  0.70\textpm 0.02 &  0.79\textpm 0.05  &  0.09\textpm 0.02 \\\midrule
\multirow{4}{*}{Adult} & Saliency    & 0.87\textpm 0.01  & 0.92\textpm 0.01  &  0.62\textpm 0.04 & 0.85\textpm 0.01  & 0.57\textpm 0.04 &  0.38\textpm 0.03 \\
                       & SmoothGrad  &  0.98\textpm 0.00 & 0.98\textpm 0.00  &0.96\textpm 0.00  & 0.98\textpm 0.00  &0.96\textpm 0.00  & 0.94\textpm 0.00 \\
                       & LIME        & 0.95\textpm 0.00 & 0.86\textpm 0.02 & 0.86\textpm 0.02 &  0.97\textpm 0.00 &  0.84\textpm 0.02 & 0.84\textpm 0.02 \\
                       & K.SHAP      & 0.88\textpm 0.01 &  0.88\textpm 0.04 &  0.65\textpm 0.04 & 0.84\textpm 0.02 & 0.47\textpm 0.10  &  0.32\textpm 0.08  \\\bottomrule    
\end{tabular}
\end{table*}

\subsection{Sensitivity Analysis}\label{sec:sensitivity}

In this section, we analyze the impact of other training hyperparameters on explanation stability, epoch-by-epoch throughout the entire training process.

Given our observation that parameter stability promotes explanation stability, we expect a lower number of epochs and a lower learning rate to both increase explanation stability under dataset shift by decreasing the divergence between optimization paths in parameter space. 
Based on the geometry of two divergent optimization paths, we anticipate that if the base model is trained for a long period of time, the retrained model's parameters that are most similar to the base model will occur relatively sooner in the retraining process.
This emphasizes the challenge of retraining models that both perform well and remain close to the base model.


We track explanations during each epoch of retraining, measuring similarity using gradient $\ell_2$ distance and top-K SA consistency. Each trial consists of varying a single hyperparameter while keeping all others fixed. For each value of the varied hyperparameter, we compare one base model with 10 models retrained on synthetically shifted HELOC data $\sim\mathcal{N}(0,0.1^2)$. Recall that \cite{hardt2016train} studied the hypothesis stability of stochastic gradient descent (SGD) and showed that practical choices such as learning rate and number of epochs of training can impact the stability of model parameters; the following experiments confirm these findings in the context of explanation stability.

\paragraph{Learning Rate}

We verify in \Cref{fig:sensitivity_lr} that increasing the learning rate of all models effectively worsens gradient stability. 
We observe that the maximal gradient stability occurs sooner when the learning rate is higher. We hypothesize that this is because a high learning rate stretches the optimization path of the base model, taking it further away from the initialization in parameter space. Since the retrained models take divergent paths from the base model (due to the dataset shifts), the likelihood that they achieve similar parameters decreases as learning rate increases.

\paragraph{Learning Rate Decay} Building upon the understanding of the effect of learning rate on gradient stability, we delve into its counterpart, learning rate decay, and its influence, as depicted in \Cref{fig:sensitivity_lrd}. The results are as expected: decreasing the decay value, i.e. increasing the amount of decay, shows similar behaviour to decreasing the learning rate. Namely, the optimization trajectories have less chance to diverge, and the resultant base model and retrained models are thus more likely to have similar parameters. While a higher amount of decay (lower decay value) is preferable, this can inhibit the training process (notice how a decay value of 0.9 causes training to stop after around 20 epochs).

\paragraph{Batch Size}

In \Cref{fig:sensitivity_batch}, we see that decreasing batch size worsens gradient stability. 
We hypothesize that this behaviour is because there are a greater number of steps in each epoch, inducing a higher level of randomness in the optimization trajectory. Each step adds an opportunity for the retrained model's path to deviate from the initial base model's path, due to the dataset shifts. As such, the likelihood of achieving similar parameters between the retrained models and the base models decreases with smaller batch sizes. Additionally, the points at which explanation stability is maximized occur after fewer epochs for smaller batch sizes, since those correspond to both longer and more divergent optimization trajectories.


\begin{figure*}[!ht]
  \centering
  \small
  \subfloat[][Sensitivity to \textbf{learning rate}.]{\includegraphics[width=0.48\textwidth]{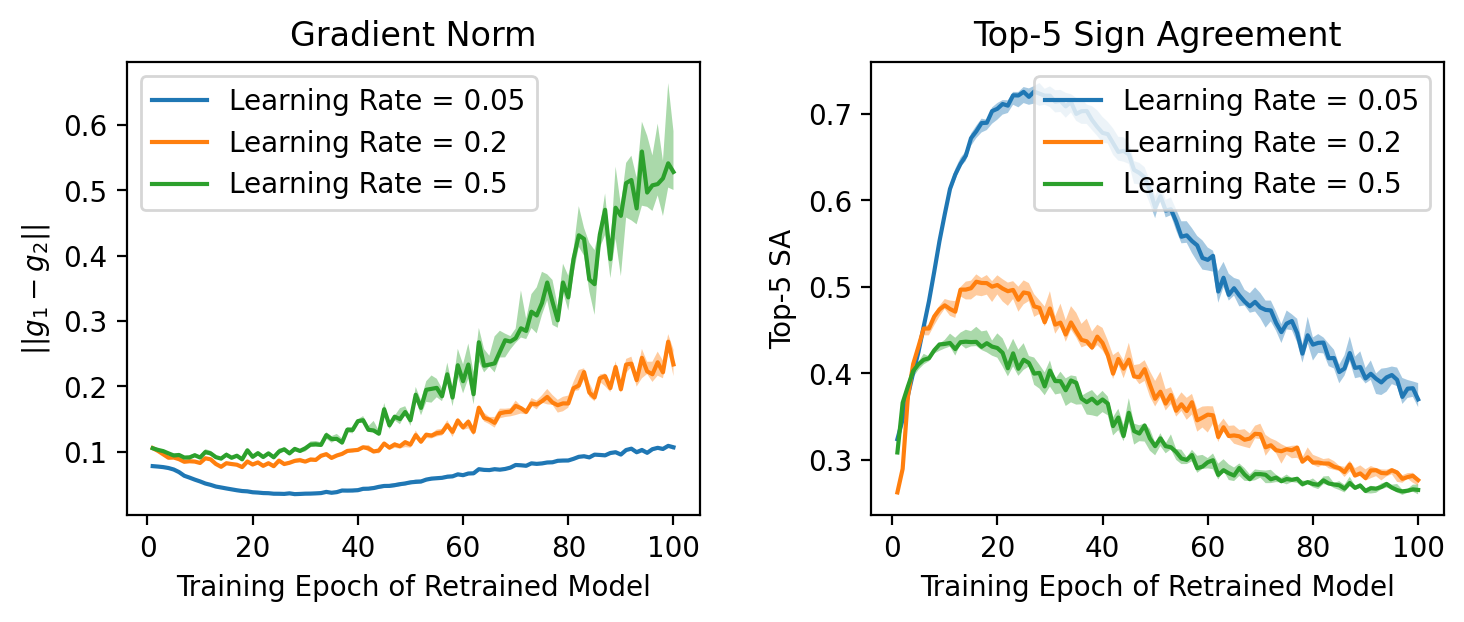}\label{fig:sensitivity_lr}}\quad
  \subfloat[][Sensitivity to \textbf{learning rate decay}.]{\includegraphics[width=0.48\textwidth]{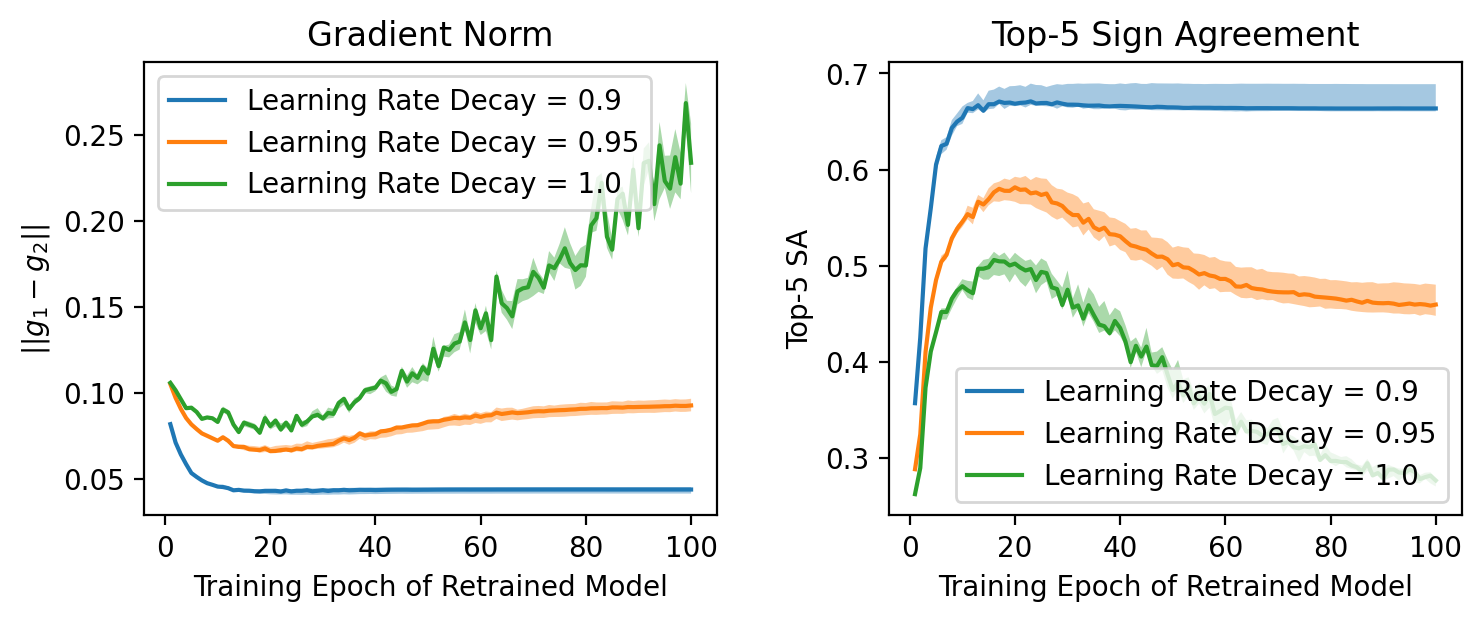}\label{fig:sensitivity_lrd}}\\
  \subfloat[][Sensitivity to \textbf{batch size}.]{\includegraphics[width=0.48\textwidth]{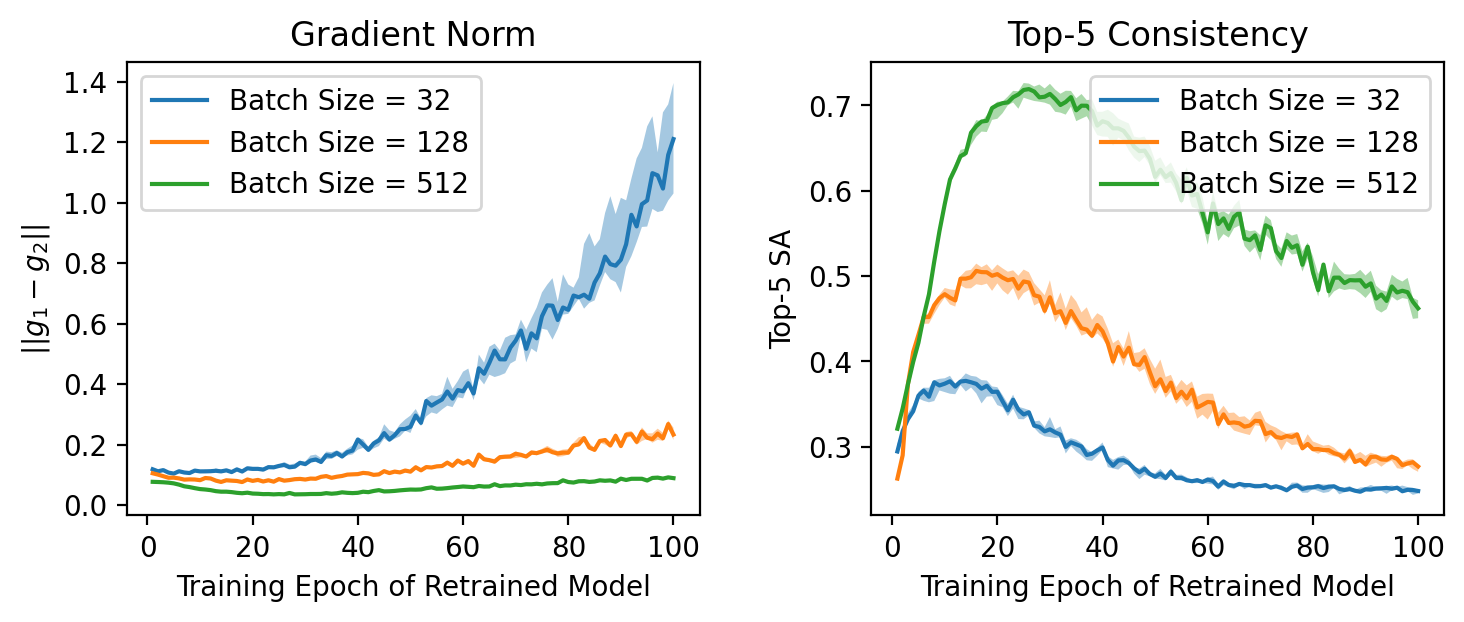}\label{fig:sensitivity_batch}}\quad
  \subfloat[][Sensitivity to the base model's number of \textbf{epochs}.]{\includegraphics[width=0.48\textwidth]{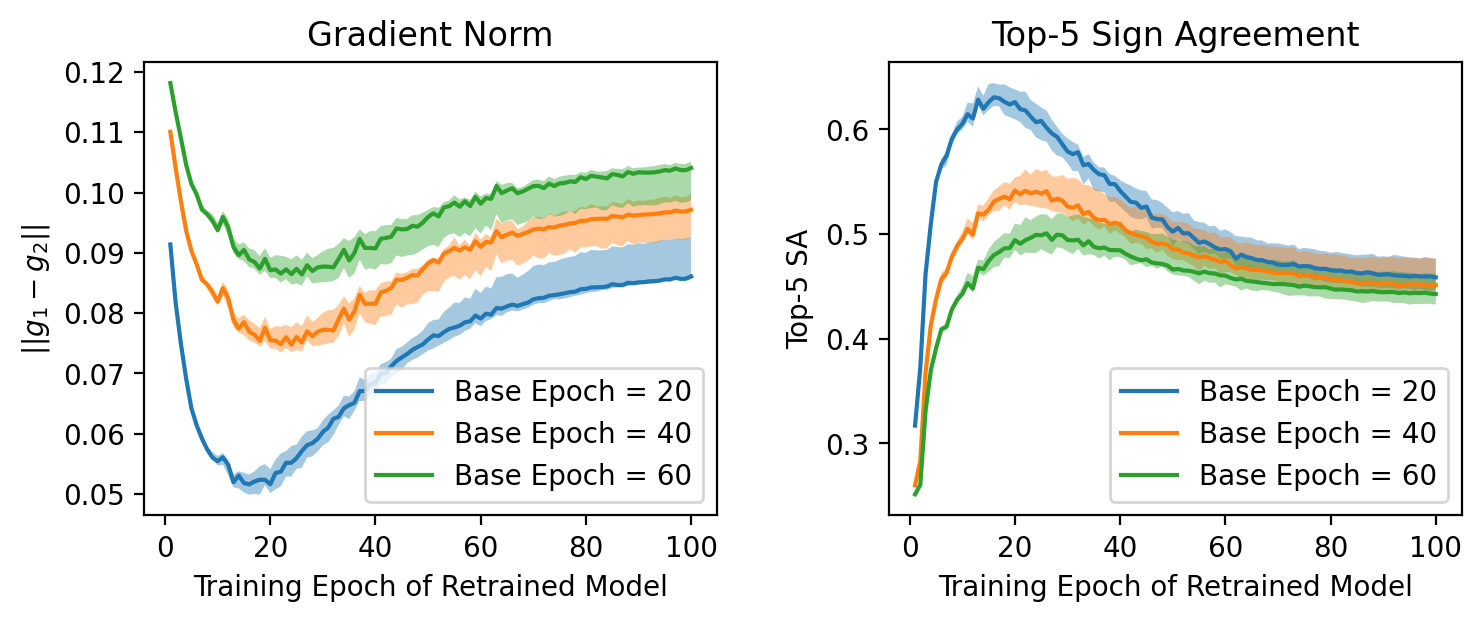}\label{fig:sensitivity_base}}
  \caption{\small Raw gradient similarity (y-axis) of retrained models are shown per epoch of training (x-axis). For gradient norm, lower values are better, while for top-5 SA, higher values (closer to 1) are better. Default values are: learning rate of 0.2; learning rate decay of 1.0 (no decay); batch size of 128; base models are trained for 30 epochs (besides the bottom right); and synthetic noise, $\sigma=0.1$.}
  \label{fig:sensitivity}
\end{figure*}

        

\paragraph{Base Model Epochs}

\Cref{fig:sensitivity_base} verifies that as the number of base model epochs is increased, explanation stability suffers, with significantly larger gradient changes and lower top-5 consistency scores. For a sufficiently small base model epoch of 20, we observe that similarity also peaks at around 20 training epochs. However, if we were to increase the base model epoch to 60, as shown in green, we observe that the peak still occurs at around the same value. This means that, after just 20 epochs, the optimization path of the retrained model must have diverged so much from the base model that any further training leads only to a decrease in stability.

\section{Conclusions}
In this paper, we theoretically and empirically characterize how modelling choices impact model explanations' stability to retraining under data shifts. Our formal results suggest, and our experimental results confirm, that model explanations can be made more stable to retraining after data shifts by using low-curvature activation functions and large weight decay values. We also took a first step towards empirically characterizing how different explanation techniques yield variable explanation stabilities under dataset shifts, noting that SmoothGrad outperforms other methods.

The message for ML practitioners is clear: in order to preserve explanations upon model retraining due to data shifts at little cost to accuracy, use softplus with the lowest possible $\beta$, and the largest weight decay that does not affect predictive accuracy. If possible, practitioners should  fine-tune the model (\S \ref{sec:finetune}) rather than retrain from scratch to avoid instability due to the multiplicity of optimal solutions. However, as shown in \S \ref{sec:sensitivity}, training hyperparameters can also play a large role in explanation stability, so care must be taken to choose hyperparameters in a way that balances accuracy and explanation stability.

There are several directions for future work.
First, our theoretical results are only applicable to differentiable models. In the appendix, we empirically evaluate explanation stability when retraining decision trees, random forests, and XGBoost models, but additional work is needed to formalize explanation stability to retraining for these and other classes of models. Second, while we do perform experiments with a temporal dataset shift, we do not evaluate other types of naturally-occurring data shifts like geographic shifts or data corrections (e.g., as with the German Credit Statlog dataset~\cite{uci-data}). Future work should focus on understanding explanation stability for various types of real-world data shifts. Third, we limit our evaluation to tabular datasets due to the commonality of these datasets where laws like a Right to Explanation are relevant. However, other sources of data may become more relevant to user explanations in the future and validating our theoretical results on non-tabular data is of independent interest.
Finally, rather than sticking to existing explanation techniques, future work can explore how to develop novel explanation techniques that preserve explanation stability despite small changes in the underlying models and data.



\begin{acknowledgements} 
    The authors would like to thank the anonymous reviewers for their helpful feedback and all the funding agencies listed below for supporting this work. This work is supported in part by the NSF awards \#IIS-2008461, \#IIS-2040989, and \#IIS-2238714, and research awards from Google, JP Morgan, Amazon, Harvard Data Science Initiative, and D\^3 Institute at Harvard. The views expressed here are those of the authors and do not reflect the official policy or position of the funding agencies.
\end{acknowledgements}

\pagebreak
\small
\bibliographystyle{abbrv}
\bibliography{refs}

\onecolumn
\appendix

\section{Proofs}

In this section, we provide the proofs for the theory presented in the main paper. We present the notations here again for convenience in Table \ref{tab:theory}. As with the main paper, we first bound parameter shift in \S \ref{subsec:paramshift} and then explanation shift in \S \ref{subsec:explnshift}.

\begin{table*}[t]
    \centering
    \caption{Notations}
    \label{tab:theory}
    \begin{tabular}{r|l|l}
        $z_i$ & $(x_i, y_i);~ x_i \in \R^d;~y \in \R$ & data point \\ 
        $\D_1$ & $\{z_1,..z_n\}$ & original training set\\
        $\D_2$ & $\{z'_1, ... z'_n \}$ & new training set\\
        $\ell_{reg}(\theta, z_i)$ & $\ell(\theta, z_i) + \gamma \| \theta \|^2_2$; ~ $\ell(\theta, z_i) \in \R^+$ & regularized loss function \\
        $\ell_{\D_1}(\theta)$ & $\frac{1}{n} \sum_{i=1}^n \ell_{reg}(\theta, z_i)$ & loss incurred by $\theta$ on dataset $D_1$ \\
        $\theta_1$ & $\arg\min_{\theta} \ell_{\D_1}(\theta)$ & optimal weights on dataset $D_1$\\
        $\epsilon_1 + \gamma \| \theta_1 \|^2$& $\ell_{\D_1}(\theta_1) = \min_{\theta} \ell_{\D_1}(\theta)$ & optimal loss value on dataset $\D_1$
    \end{tabular}
\end{table*}

\subsection{Bounding the parameter shift}\label{subsec:paramshift}
In order to bound the parameter shift, we first prove an intermediate result connecting the loss on $\D_2$ of the optimum obtained by minimizing the loss on $\D_1$. 

\begin{lemma}\label{lemma:loss}
The loss $\ell_{\D_2}$ on $\theta_1$ is given by

\begin{align*}
    \ell_{\D_2}(\theta_1) \leq \epsilon_1 + \gamma \| \theta_1 \|^2 + \mathcal{L}_x(\theta_1) d(\D_1, \D_2) 
\end{align*}

\end{lemma}

\begin{proof}

\begin{align*}
    \ell_{\D_2}(\theta_1) &=\ell_{\D_1}(\theta_1) + \left( \ell_{\D_2}(\theta_1) - \ell_{\D_1}(\theta_1) \right) \\
    & \leq \ell_{\D_1}(\theta_1) + | \ell_{\D_2}(\theta_1) - \ell_{\D_1}(\theta_1) | \\
    & \leq \epsilon_1 + \gamma \p \theta_1 \p^2 +   \frac{1}{n} \sum_{i=1}^n | \ell(\theta_1, z_i) - \ell(\theta_1, z'_i) | ~~~~ \text{(Convexity of abs fn)}\\
    & \leq \epsilon_1 + \gamma \p \theta_1 \p^2 + \frac{1}{n} \sum_{i=1}^n \mathcal{L}_x(\theta_1) \p z_i - z'_i \p_2 ~~~~ \text{(Defn of Lipschitz)}\\
    &\leq \epsilon_1 + \gamma \p \theta_1 \p^2 + \mathcal{L}_x(\theta_1) \underbrace{\frac{1}{n} \sum_{i=1}^n \p z_i - z'_i \p}_{d(\D_1, \D_2)} \\
\end{align*}

\end{proof}

In the result above, $\epsilon_1 + \gamma \| \theta_1 \|^2$ is the optimal value of the regularized loss, and the term following that denotes the shift from the optimal value incurred. As mentioned in the paper, we use the following additional assumptions to derive the main result: (1) we minimize a regularized loss $\ell_{reg}$ that involves weight decay, i.e, $\ell_{reg}(\theta) = \ell(\theta) + \gamma \| \theta \|^2_2$; (2) $\ell$ is locally quadratic; (3) the learning algorithm returns a unique minimum $\theta$ given a dataset $\D$.

\begin{theorem} Given the assumptions stated above, and that $\mathcal{L}_x (\theta_1)$ is the Lipschitz constant of the model with parameters $\theta_1$, we have 

\begin{align*}
    \p \theta_2 - \theta_1 \p_2 & \leq \sqrt{\frac{ \mathcal{L}_x(\theta_1) d(\D_1, \D_2) }{\gamma }} + C
\end{align*}

where $\gamma$ is the weight decay regularization constant, and $C$ is a small problem-dependent constant.

\end{theorem}

\begin{proof}

Assuming the loss $\ell_{D_2}$ is locally quadratic around the optimal solution $\theta_2$, we can employ a second order Taylor series expansion as follows: 

\begin{align}
    \ell_{D_2}(\theta_1) &= \ell_{D_2}(\theta_2) + \nabla_{\theta} \ell_{D_2}(\theta_2)^\top (\theta_1 - \theta_2) + \frac{1}{2} (\theta_1 - \theta_2)^\top H_{\theta_2} (\theta_1 - \theta_2) \\
    &= \epsilon_2 + \gamma \p \theta_2 \p_2^2 + \frac{1}{2} (\theta_1 - \theta_2)^\top H_{\theta_2} (\theta_1 - \theta_2) ~~~~ (\text{global optimality of $\theta_2$ on $S_2$})
\end{align}
In the statement above we have used the fact that $\nabla_{\theta} \ell_{\D_2}(\theta_2) = 0$ because of first order optimality criteria, and that $\epsilon_2 + \gamma \| \theta_2 \|^2 = \ell_{\D_2}(\theta_2)$, is the optimal value. Now we use Lemma \ref{lemma:loss} to substitute $\ell_{\D_2}(\theta_1)$, and we have:

\begin{align}
    \epsilon_2 + \gamma \p \theta_2 \p_2^2 + \frac{1}{2} (\theta_1 - \theta_2)^\top H_{\theta_2} (\theta_1 - \theta_2) &\leq \epsilon_1 + \gamma \p \theta_1 \p_2^2 + \mathcal{L}_x(\theta_1) d(\D_1, \D_2)  \\
    \frac{1}{2} \lambda_{\min}(H_{\theta_2}) \p \theta_1 - \theta_2 \p^2_2 &\leq (\epsilon_1 - \epsilon_2) + \gamma (\p \theta_1 \p_2^2 - \p \theta_2 \p_2^2) + \mathcal{L}_x(\theta_1) d(\D_1, \D_2) \label{eqn:4}
\end{align}

To derive the second step we have used the following fact that a lower bound on the quadratic term is given by the lowest eigenvalue i.e., $\arg\min_{\theta; \| \theta \| = 1} \theta^{\top} H \theta = \lambda_{\min} \implies \theta^{\top} H \theta \geq \lambda_{\min} \| \theta \|^2$. We now observe the following relationship for $\lambda_{min}$:

\begin{align}
    \lambda_{\min}(\nabla_{\theta}^2 \ell_{reg}(\theta_2)) =  \lambda_{\min}(\nabla_{\theta}^2 \ell(\theta_2)) + 2 \gamma \geq 2 \gamma \label{eqn:5}
\end{align}

as $\nabla_{\theta}^2 \ell(\theta_2)$ is positive definite due to global optimality. Putting eqn. \ref{eqn:5} in eqn. \ref{eqn:4}, we have the final result. Here, the ``problem dependent constant'' $C = \sqrt{\frac{\ell_{\D_1}(\theta_1) - \ell_{\D_2}(\theta_2)}{\gamma}}$, which we can expect to be quite small in practice if the optimal values $\ell_{\D_1}(\theta_1) \approx \ell_{\D_2}(\theta_2)$. Thus in practice we can assume $C$ to be very small.

\end{proof}

\subsection{Bounding the explanation shift}\label{subsec:explnshift}

\begin{lemma}
The parameter-input Lipschitz has the following property: 

\begin{align*}
    \E_{\X \sim \D} \p \nabla_x f(x; \theta_1) - \nabla_x f(x; \theta_2) \p_2 \leq~ \E_{\theta \in \Theta} \E_{\X \in \D} \| \nabla_{\theta} \grad f(x, \theta) \|_2 \times  \p \theta_2 - \theta_1 \p_2
\end{align*}

where $\Theta = \{ \theta_{\lambda} = \lambda \theta_2 + (1 - \lambda) \theta_1 \mid \lambda \in [0,1] \}$

\end{lemma}

\begin{proof}
The proof follows from the fundamental theorem of integral calculus. Let $g(\theta, \X) = \grad f(\X, \theta)$ for convenience.

\begin{align*}
     g(\theta_2, \X) - g(\theta_1, \X) &= \left(\int_{\lambda=0}^{1} \nabla_{\theta} g(\theta_{\lambda}, \X) d\lambda \right)^{\top} (\theta_2 - \theta_1) \\ 
     \| g(\theta_2, \X) - g(\theta_1, \X) \|_2 &\leq  \| \int_{\lambda=0}^{1}  \nabla_{\theta} g(\theta_{\lambda}, \X) d\lambda \|_2 \| \theta_2 - \theta_1 \|_2 ~~~~~ (\text{Cauchy Schwartz})\\
     &\leq \left(\int_{\lambda=0}^{1}  \| \nabla_{\theta} g(\theta_{\lambda}, \X)  \|_2 d\lambda \right) \| \theta_2 - \theta_1 \|_2 ~~~~ (\text{Jensen's inequality})\\
     \E_{\X \in \D} \| g(\theta_2, \X) - g(\theta_1, \X) \|_2 &\leq \left(\int_{\lambda=0}^{1}  \E_{\D} \| \nabla_{\theta} g(\theta_{\lambda}, \X)  \|_2 d\lambda \right) \| \theta_2 - \theta_1 \|_2 ~~~~ (\text{Swapping expectation and integral})\\
     \E_{\X \in \D} \| g(\theta_2, \X) - g(\theta_1, \X) \|_2 &\leq \E_{\Theta} \E_{\D} \| \nabla_{\theta} g(\theta, \X)  \|_2 \| \theta_2 - \theta_1 \|_2 ~~~~ (\text{defn of expectation})
\end{align*}

\end{proof}

\begin{theorem}
Assume that a 1-hidden layer neural network with weights $\theta$, and random inputs $\X \sim \mathcal{N}(0, I)$\footnote{Covariance of I is chosen for notational brevity}. Further assume that we have use an activation function $\sigma$ with well-defined second derivatives (e.g: softplus). For this case, the parameter-input derivatives have the following form:

\begin{align*}
   E_{\X} \| \nabla_{\theta} \grad \ell(\X, \theta) \|_2 \leq E_{\X} \| \nabla_{\theta} \grad \ell(\X, \theta) \|_F \leq \| \theta \|_2 + \beta \phi(\theta)
\end{align*}

where $\beta$ is an maximum curvature, and $\phi(\theta)$ is the path-norm \cite{neyshabur2015norm} of the model.

\end{theorem}

\begin{proof}
We derive this expression for the case of a scalar valued neural network
$f(\X) = W_2 \sigma(W_1 \X)$, where $W_2 \in \R^{h \times d}, W_1 \in \R^{1 \times h}$, where $h$ is the number of hidden layers, and $d$ is the input dimensionality.
Here, $\sigma: \R \rightarrow \R$ is a point-wise smooth non-linearity with a well-defined curvature, such as softplus, but not ReLU. For such a non-linearity, 
let $\sigma''(x) \vcentcolon= \frac{\partial^2 \sigma(x)}{\partial x^2 } \leq M$. For softplus, this $M = \beta$ hyper-parameter is implicit in its definition \cite{srinivas2022efficient}. 

By straightforward calculus, its gradient and the ``parameter-input derivatives'' are given by 

\begin{align*}
    \frac{\partial f(\X)}{\partial \X_i} &= \sum_{k=1}^{h} W_2^k \sigma'(W_1 \X)^k W_1^{i,k} \\
    \frac{\partial^2 f(\X)}{\partial \X_i \partial W_2^j} &= \sigma'(W_1 \X)^j W_1^{i,j} \\
    \frac{\partial^2 f(\X)}{\partial \X_i \partial W_1^{i,j}} &= W^j_2 \sigma'(W \X)^j + W^j_2 \sigma''(W_1 \X)^j W_1^{i,j} \X_i \\ 
    \frac{\partial^2 f(\X)}{\partial \X_i \partial W_1^{m,j}} &= W^j_2 \sigma''(W_1 \X)^j W_1^{i,j} \X_m ~~~~~(\forall m \neq i)\\ 
\end{align*}

Now, assume that $\X \sim \mathcal{N}(0, I)$, i.e., the input is an independent normal distribution. Let us now compute the squared terms for the parameter-input derivatives to eventually be able to compute its Frobenius norm.

\begin{align*}
    \left( \frac{\partial^2 f(\X)}{\partial \X_i \partial W_2^j} \right)^2 &\leq (\sigma'(W_1 \X)^j W_1^{i,j})^2 \leq (W_1^{i,j})^2 ~~~~(\sigma' \leq 1 ~ \text{for softplus})\\
    \E_{\X}\left(\frac{\partial^2 f(\X)}{\partial \X_i \partial W_1^{i,j}}\right)^2 &= \E_{\X} (W^j_2 \sigma'(W \X)^j + W^j_2 \sigma''(W_1 \X)^j W_1^{i,j} \X_i)^2 \\ 
    & \leq (W^j_2)^2 + (W^j_2 \sigma''(W_1 \X)^j W_1^{i,j})^2 ~~~~ (\text{Using} ~\X \sim \mathcal{N}(0,I))\\
    & \leq (W^j_2)^2 + \beta^2 (W^j_2  W_1^{i,j})^2 ~~~~ (\sigma'' \leq \beta)\\
    \E_{\X} \left( \frac{\partial^2 f(\X)}{\partial \X_i \partial W_1^{m,j}} \right)^2 &\leq \beta^2 (W^j_2 W_1^{i,j})^2 ~~~~~(\forall m \neq i)\\ 
\end{align*}

Computing the Frobenius norm, we have 

\begin{align*}
    \E_{\X} \| \nabla_{\theta} \grad f(\X) \|_F^2 &= \sum_{i,j,k} \E_{X} \left( \frac{\partial^2 f(\X) }{\partial \X_i \partial W_1^{j,k}} \right) + \sum_{i,j} \E_{X} \left( \frac{\partial^2 f(\X) }{\partial \X_i \partial W_2^j} \right) \\
    & \leq \left( \sum_{j,k} (W_2^k)^2 + \sum_{j=i} \beta^2 (W_2^k W_1^{j,k})^2 + \sum_{j \neq i} \beta^2 (W_2^k W_1^{j,k})^2 \right) + \sum_{i,j} (W_1^{i,j})^2 \\
    & \leq  \left( \sum_k (W_2^k)^2 + \sum_{i,j} (W_1^{i,j})^2 \right) + \beta^2 \sum_{j,k} (W_2^k W_1^{j,k})^2 \\
    & \leq \| \theta \|^2 + \beta^2 \phi^2(\theta)
\end{align*}

Here, we use the fact that $\phi(\theta) = \sqrt{\sum_{j,k} (W_2^k W_1^{j,k})^2} $ is the 2-path-norm \cite{neyshabur2015norm} for the purpose of characterizing generalization in neural networks. Finally, we re-write the above using square roots:

\begin{align*}
    \E_{\X} \| \nabla_{\theta} \grad f(\X) \|_F &\leq \sqrt{ \E_{\X}  \| \nabla_{\theta} \grad f(\X) \|_F^2 } \\
    &\leq \sqrt{\| \theta \|_2^2 + \beta^2 \phi^2(\theta)}\\
    &\leq \| \theta \|_2 + \beta \phi(\theta) 
\end{align*}
\end{proof}

\section{Additional experiments}

\paragraph{Additional data from \S 4.1}

The specific hyperparameters that we use to a) train base models, and b) fine-tune these models, are shown in \Cref{tab:fine_tuning_hyperparams}. These hyperparameters are chosen such that a minimum can be attained by the optimizer (SGD), as is required to verify the theory. The learning rates shown in the fine-tuning process for HELOC and Adult (indicated by a star in \Cref{tab:fine_tuning_hyperparams}) are multiplied by a scalar that is dependent on the amount of noise added to the dataset. Intuitively, adding more noise causes a larger increase in loss, so we increase the learning rate used during fine-tuning in order to converge to the new minimum within 50 epochs. The alternative approach, used in the WHO experiments, is to use more epochs during fine-tuning, and instead start with a constant learning rate that is higher but decays more, such that we again see convergence.

\begin{table}[t]
    \small
    \centering
    \caption{\small Details of (dataset specific) hyperparameters used in \S 4.1 fine-tuning. Base models are trained for a large enough number of epochs to ensure convergence to a minimum. Decay value indicates the amount of decay applied to the learning rate, and step size indicates the number of epochs between applying decay.}
    \begin{tabular}{r|cccccc}\toprule
        & \multicolumn{2}{c}{WHO} & \multicolumn{2}{c}{HELOC} & \multicolumn{2}{c}{Adult}\\ 
        & Base & Fine-Tune & Base & Fine-Tune & Base & Fine-Tune 
         \\\midrule
        Epochs & 4000 & 1000 & 500 & 50 & 500 & 50\\
        Learning Rate & 0.5 & 0.1 & 0.5 & 0.01* & 0.5 & 0.01*\\
        Decay Value & 0.9 & 0.9 & 0.95 & 0.95 & 0.95 & 0.95\\
        Step Size & 40 & 40 & 100 & 100 & 100 & 100
        \\\bottomrule
    \end{tabular}
    
    * these learning rates vary as described in this appendix.
    \label{tab:fine_tuning_hyperparams}
\end{table}

Accuracy data for WHO in the fine-tuning experiments is shown in \Cref{tab:acc1}. Larger weight decay values and lower model curvatures do reduce train accuracy, as these modifications make it more difficult to overfit the model in reaching an exact minimum. However, test accuracy remains very similar across techniques. For the experiments that add synthetic noise, the final accuracy after fine-tuning is consistent regardless of how much noise we add, up to a standard deviation of 0.1. The accuracy is in the mid-70\% range for HELOC and the mid-80\% range for Adult. 

\Cref{fig:fine_tune_adult} shows the effect of weight decay and curvature on the Adult dataset for the fine-tuning experiments (omitted in the main text). In the case of curvature, the trend towards higher gradient stability given less curvature is not very pronounced, which we attribute to the difficulty associated with training a model to an exact minimum. In the presence of larger noise, the closest minimum in the loss landscape during fine-tuning may also shift to a different minimum.

\begin{table}[t]
    \small
    \centering
    \caption{\small Accuracy for fine-tuning in the WHO dataset. Confidence intervals are the middle 50\% of values. For softplus, $\gamma=0.001$.}
    \begin{tabular}{cccccccccc}\toprule
        \multicolumn{2}{c}{ReLU / $\gamma$=0} & \multicolumn{2}{c}{ReLU / $\gamma$=0.001} & \multicolumn{2}{c}{ReLU / $\gamma$=0.01} & \multicolumn{2}{c}{SP / $\beta=10$} & \multicolumn{2}{c}{SP / $\beta=5$} \\ 
        Train & Test & Train & Test & Train & Test & Train & Test & Train & Test 
         \\\midrule
        99.1 \textpm 0.1 & 92.8\textpm 0.5  & 98.9 \textpm 0.2 & 92.9 \textpm 0.4 & 96.5\textpm 0.2 & 92.7 \textpm 0.3 & 97.2 \textpm 0.2 & 93.2 \textpm  0.2 & 95.3\textpm 0.1 & 91.8\textpm 0.3
        \\\bottomrule
    \end{tabular}
    \label{tab:acc1}
\end{table}

\begin{figure}
    \centering
    \includegraphics[width=0.48\textwidth]{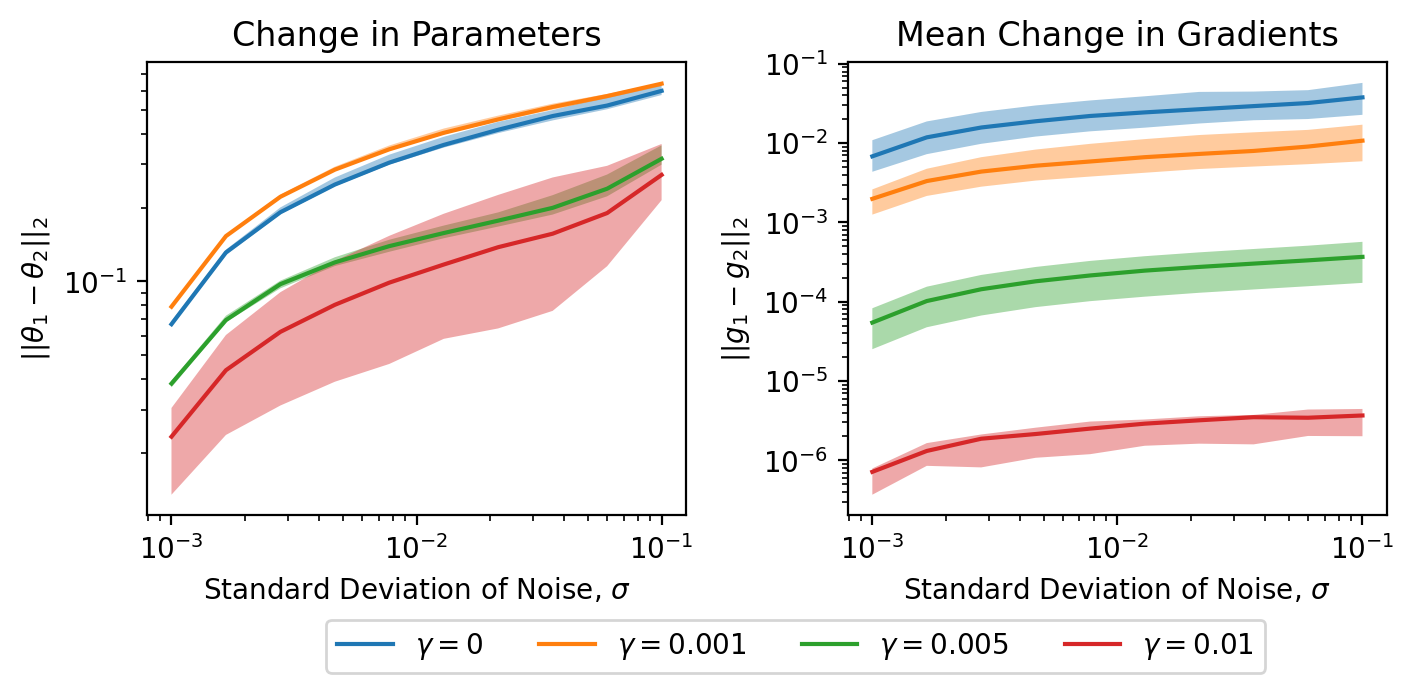}
    \includegraphics[width=0.48\textwidth]{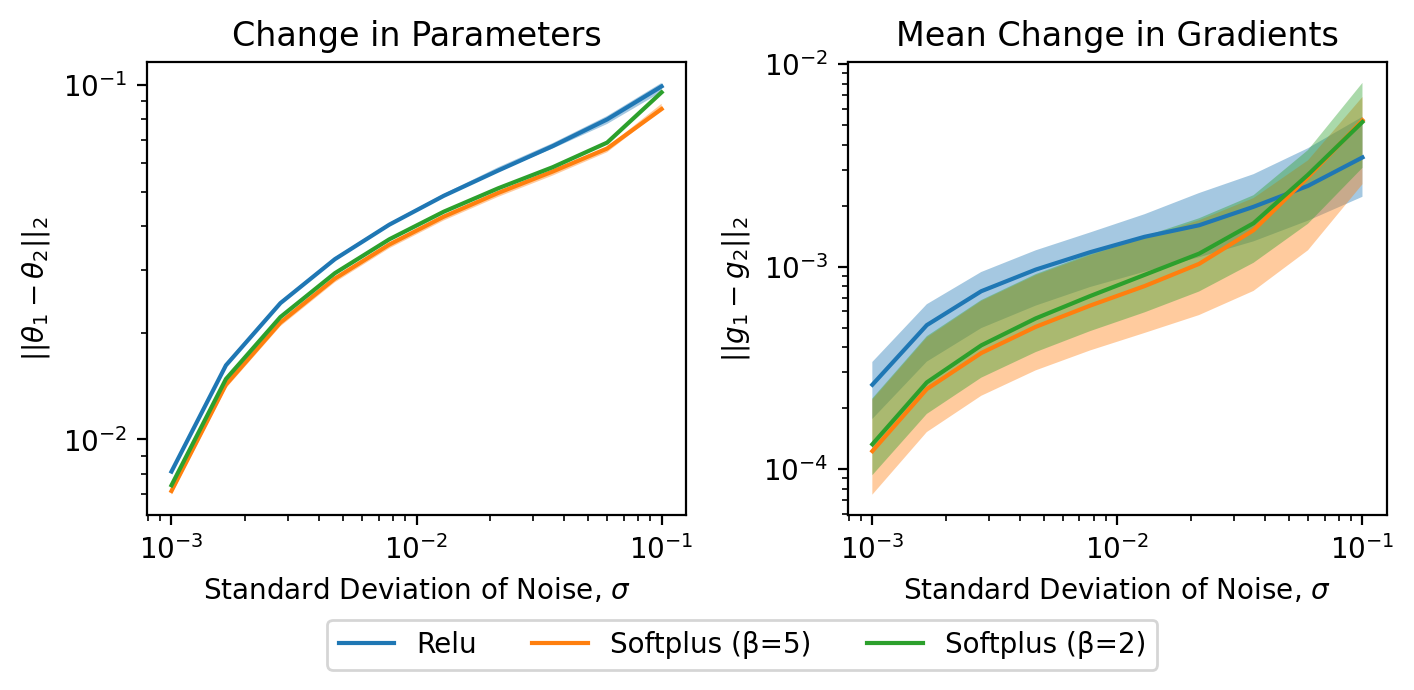}
    \caption{\small Effect of weight decay value (left) and model curvature (right) on gradient and parameter stability for the Adult dataset. All models use ReLU activation. Shifted models were trained via fine-tuning the original model. The x-axis is the size of the data shift, represented by the standard deviation $\sigma$ of Gaussian $\mathcal{N}(0,\sigma^2)$ noise. }
    \label{fig:fine_tune_adult}
\end{figure}

\paragraph{Additional data from \S 4.2}

For a given dataset, the hyperparameters for the retraining experiments are fixed between base models and retrained models in this section. In the HELOC and Adult datasets, we train any particular model for 30 epochs with a learning rate of 0.2, and no decay. For the WHO dataset, we use 80 epochs, an initial learning rate 0.8, with decay value 0.8 and step size 10 (i.e. the learning rate is multiplied by 0.8 every 10 epochs).

Accuracy data for the retraining experiments is in \Cref{tab:acc2} for WHO, and \Cref{fig:acc2helocadult} for HELOC and Adult. In the latter, the base model accuracy is essentially identical to the lower extreme on the x-axis (when very little/no noise has been added to the training data). In all datasets, while there exists some disparity between training and test data, this is much smaller than in the fine-tuning experiments. Importantly, test accuracies remain approximately constant across model types (e.g. lower or higher weight decay or curvature). We see a slight drop in accuracy for the Adult dataset using the softplus model with $\beta=2$. Despite how similar the test accuracies of all shifted models are, we observe in the main text and in the following figures that gradient stability across test inputs can still vary greatly, independent of test accuracy.

The left and right sections of \Cref{fig:retraining_gamma_curvature_params} show how parameters change given different weight decays and curvatures, respectively, as the data shift grows. This illustrates the same trends that we see in the main text, namely, that increasing weight decay or reducing curvature results in a smaller change in parameters, while adding more noise to the data results in a larger change.

\begin{figure}[t]
    \centering
    \includegraphics[width=0.95\textwidth]{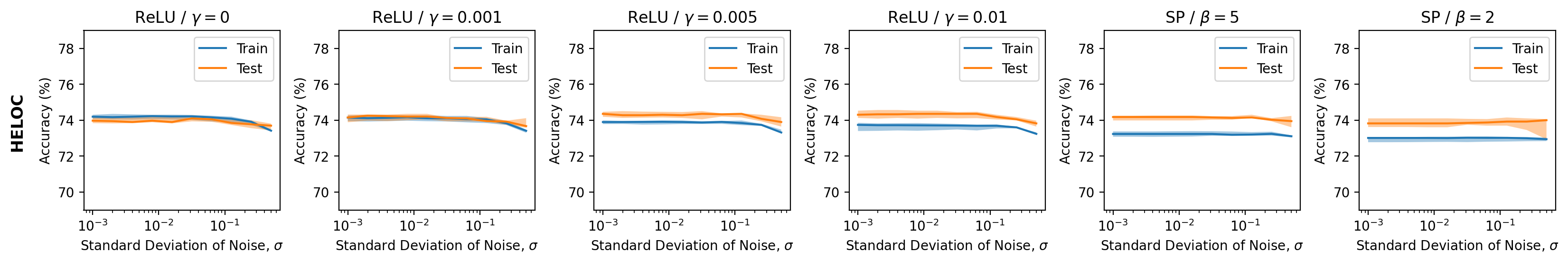}
    \includegraphics[width=0.95\textwidth]{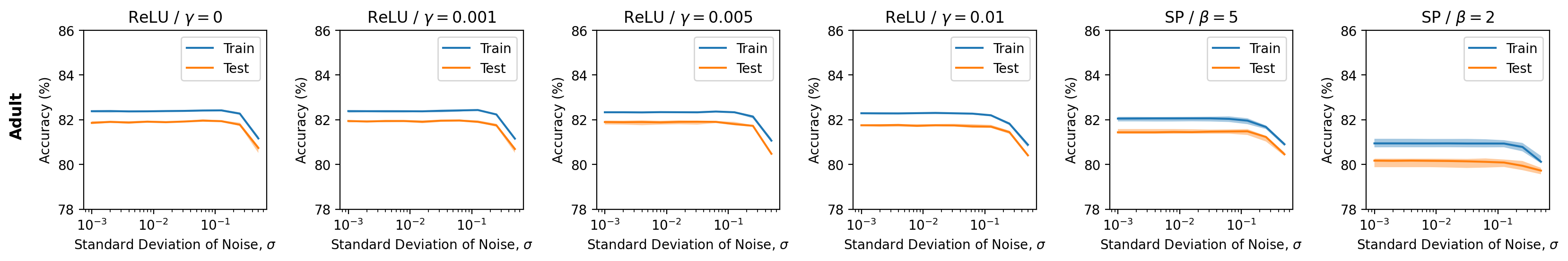}
    \caption{\small Accuracy for retraining (HELOC and Adult datasets). Confidence intervals are the middle 50\% of values.}
    \label{fig:acc2helocadult}
\end{figure}

\begin{table}[b]
    \small
    \centering
    \caption{\small Accuracy for of base models (upper row) and shifted models (lower row) in retraining experiments. Confidence intervals are the middle 50\% of values. For softplus, $\gamma=0.001$.}
    \begin{tabular}{l|cccccccccc}\toprule
        & \multicolumn{2}{c}{ReLU / $\gamma$=0} & \multicolumn{2}{c}{ReLU / $\gamma$=0.001} & \multicolumn{2}{c}{ReLU / $\gamma$=0.01} & \multicolumn{2}{c}{SP / $\beta=10$ large$^*$} & \multicolumn{2}{c}{SP / $\beta=5$} \\ 
        & Train & Test & Train & Test & Train & Test & Train & Test & Train & Test 
         \\\midrule
        Base & 98.3 \textpm 0.2 & 93.3 \textpm 0.2 & 97.7\textpm 0.2 & 93.9\textpm 0.1 & 94.7\textpm 0.0 & 93.4\textpm 0.4 & 95.8 \textpm 0.0 & 92.9 \textpm 0.4 & 94.5\textpm 0.0 & 91.5\textpm 0.2 \\
        Shifted & 98.0 \textpm 0.2 & 93.7 \textpm 0.4 & 97.4\textpm 0.2 & 93.7\textpm 0.4 & 94.7\textpm 0.2 & 92.5\textpm 0.3 & 95.7 \textpm 0.0 & 93.2\textpm 0.2  & 94.7 \textpm 0.1 & 92.6 \textpm 0.2
        \\\bottomrule
    \end{tabular}
    \label{tab:acc2}
\end{table}

\begin{figure}[t]
    \centering
    \includegraphics[width=0.48\textwidth]{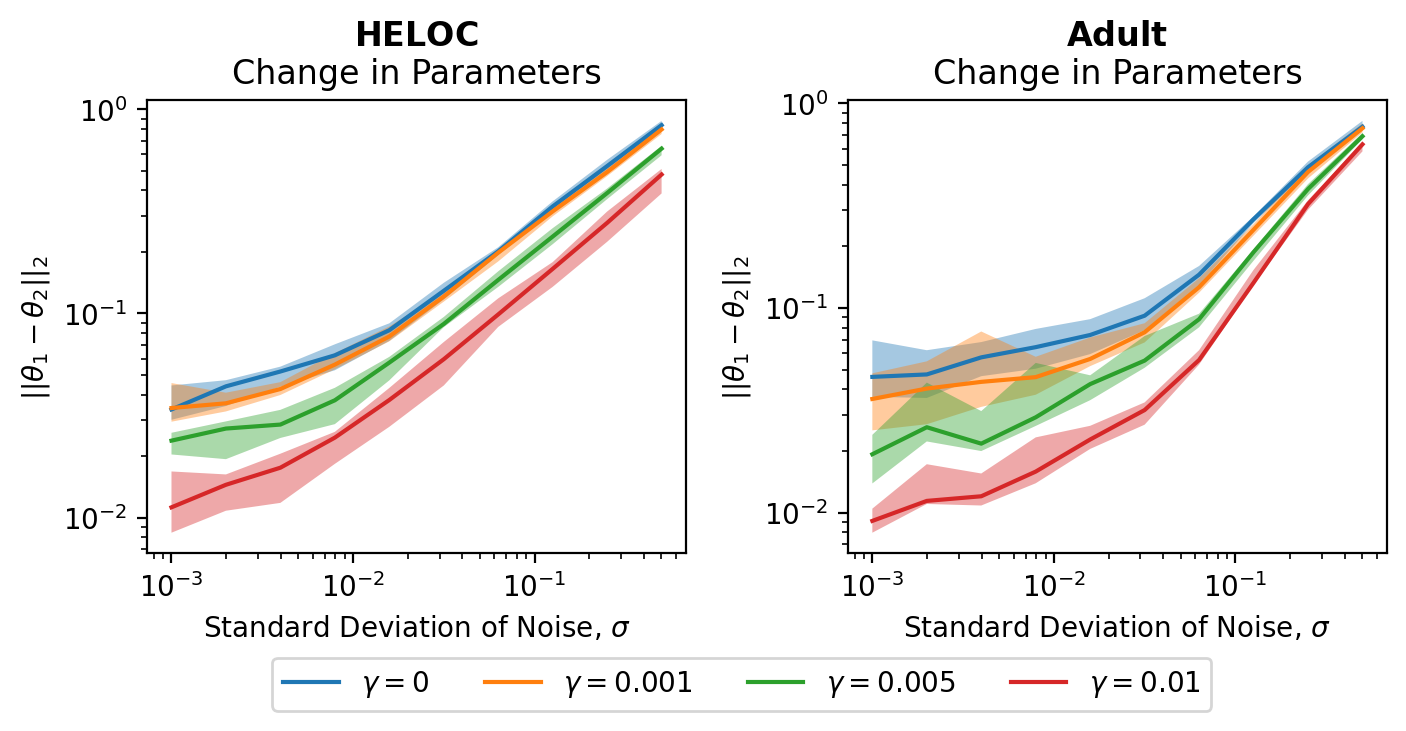}
    \includegraphics[width=0.48\textwidth]{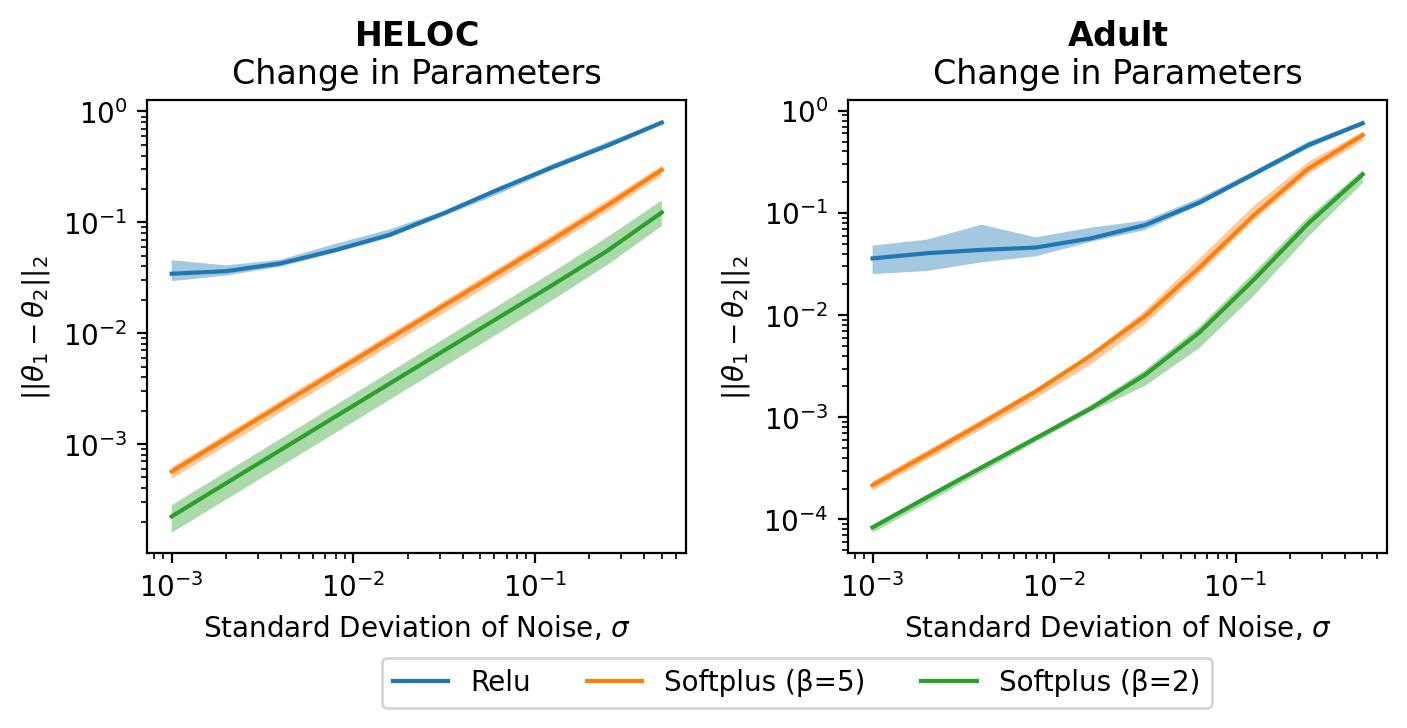}
    \caption{\small Effect of weight decay (left) and model curvature (right) on parameter stability when retraining on the HELOC and Adult datasets. The x-axis is the size of the data shift, represented by the standard deviation $\sigma$ of $\mathcal{N}(0,1)$ noise.}
    \label{fig:retraining_gamma_curvature_params}
\end{figure}

Figures~\ref{fig:explanations_retraining_full} and~\ref{fig:explanations_retraining_full2} show the Top-5 consistency scores of HELOC and Adult, respectively, each for Saliency and SmoothGrad techniques. We show the same for LIME and K.SHAP in \Cref{fig:explanations_retraining_full_heloc} for HELOC and \Cref{fig:explanations_retraining_full_adult} for Adult. Overall, SmoothGrad demonstrates consistently strong performance with respect to explanation stability, while LIME and K.SHAP show poorer performance with higher variability. For all explanation techniques, we see the same consistent trends with regards to weight decay and curvature, though the effect of each of these is occasionally less distinct in the cases of LIME and K.SHAP (i.e. when the explanation technique is inherently less stable, which we attribute mostly to the number of samples used in LIME, the countermeasures we propose are sometimes less effective).

\begin{figure*}[t]
    \centering
    \includegraphics[width=0.95\textwidth]{figures/retraining_top5_SA_heloc.png}
    \includegraphics[width=0.95\textwidth]{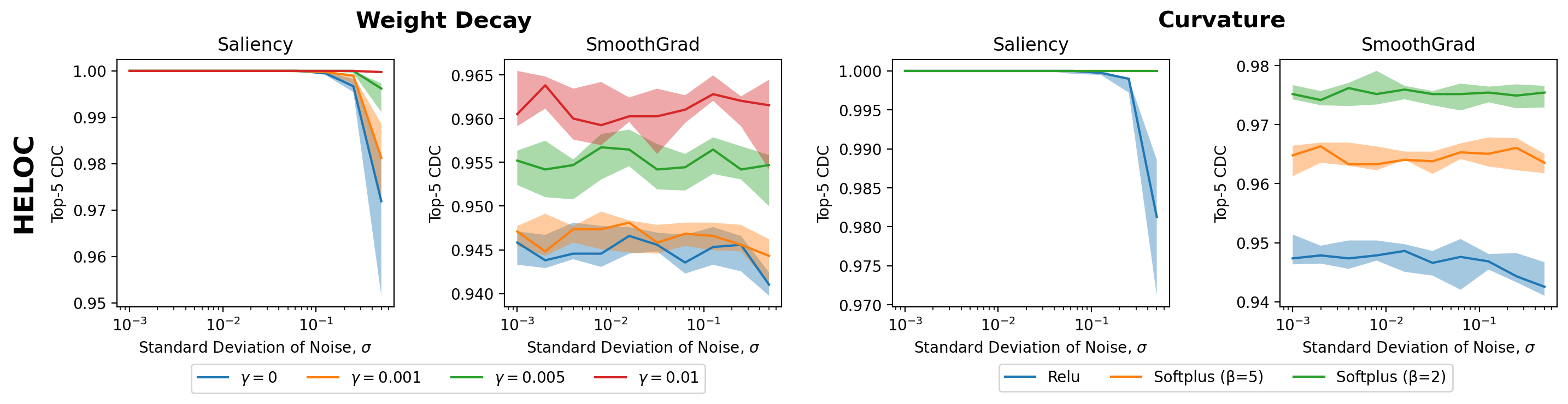}
    \includegraphics[width=0.95\textwidth]{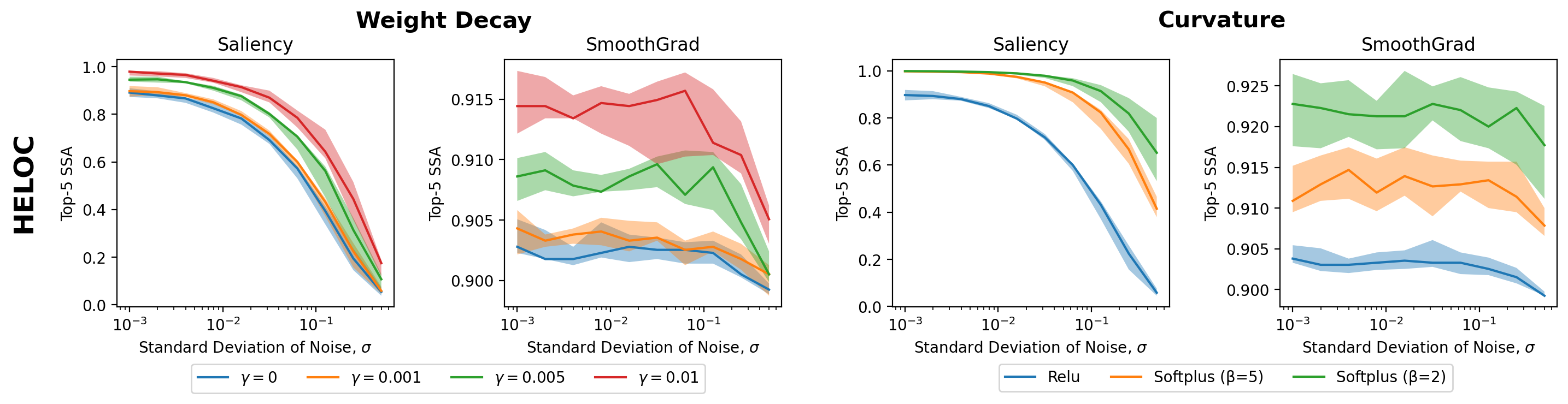}
    \caption{\small Top-5 consistency. From the top, HELOC SA, HELOC CDC, HELOC SSA (the graph for HELOC SA appears in the main text). Each row shows the effects of weight decay and curvature as data shift grows for salience and SmoothGrad top-5 metrics. Confidence intervals represent the middle 50\% of values.}
    \label{fig:explanations_retraining_full}
\end{figure*}

\begin{figure*}[t]
    \centering
    \includegraphics[width=0.95\textwidth]{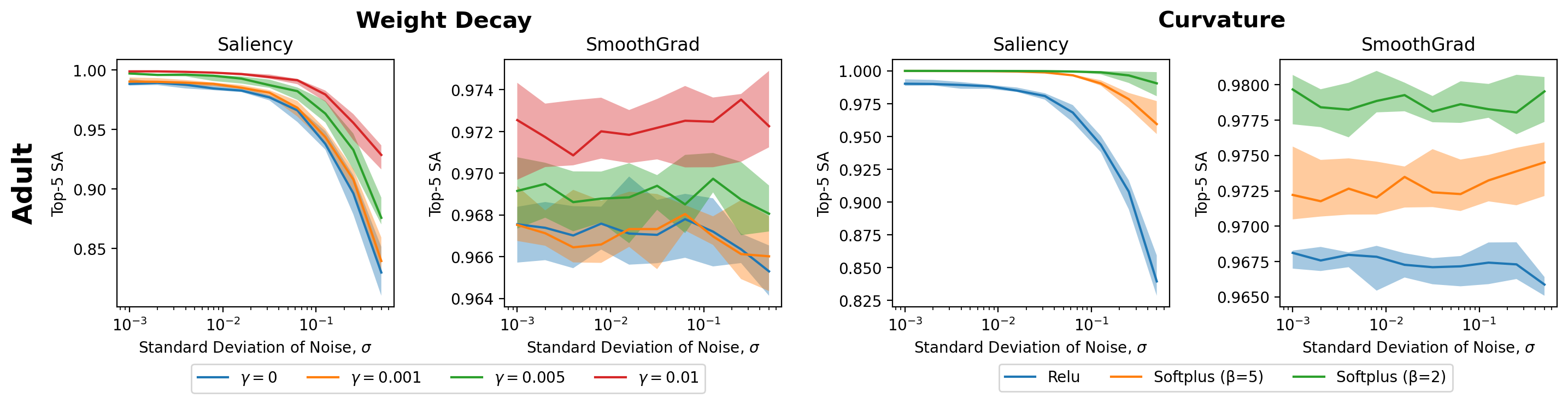}
    \includegraphics[width=0.95\textwidth]{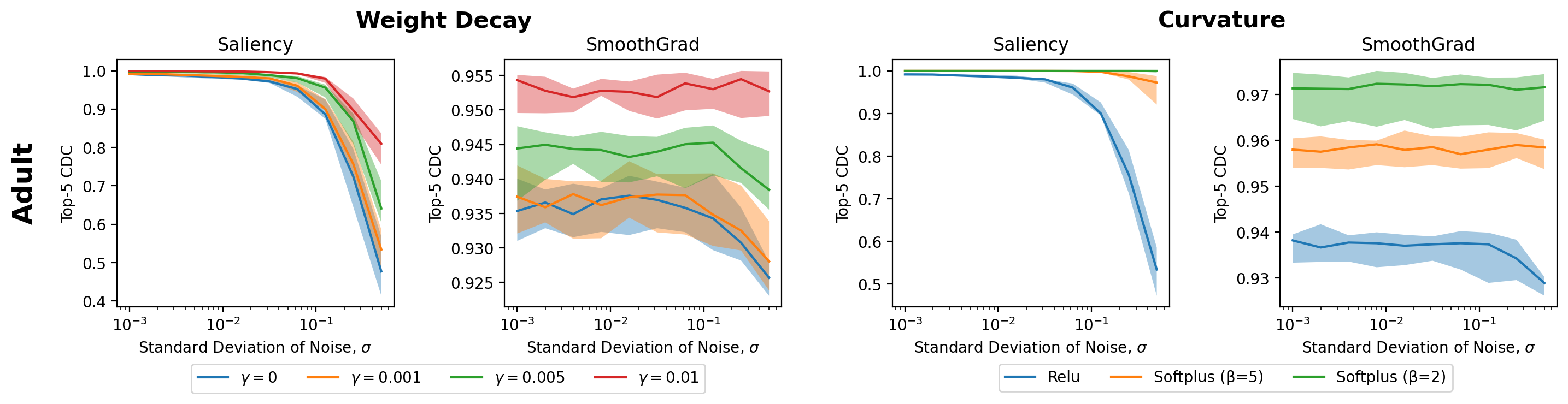}
    \includegraphics[width=0.95\textwidth]{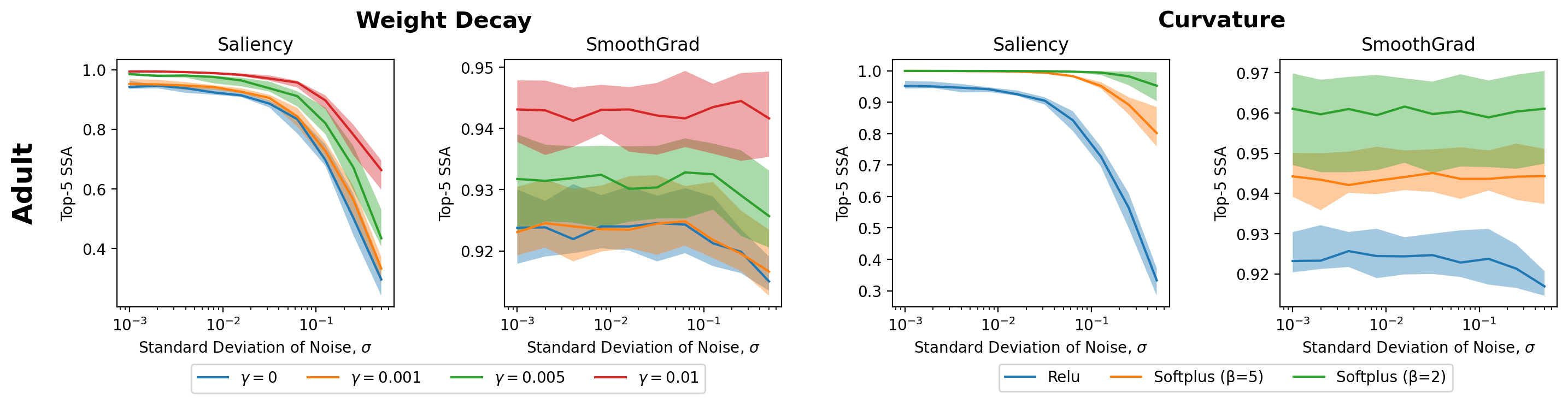}
    \caption{\small Top-5 consistency. From the top, Adult SA, Adult CDC, Adult SSA. Each row shows the effects of weight decay and curvature as data shift grows for salience and SmoothGrad top-5 metrics. Confidence intervals represent the middle 50\% of values.}
    \label{fig:explanations_retraining_full2}
\end{figure*}

\begin{figure*}[t]
    \centering
    \includegraphics[width=0.95\textwidth]{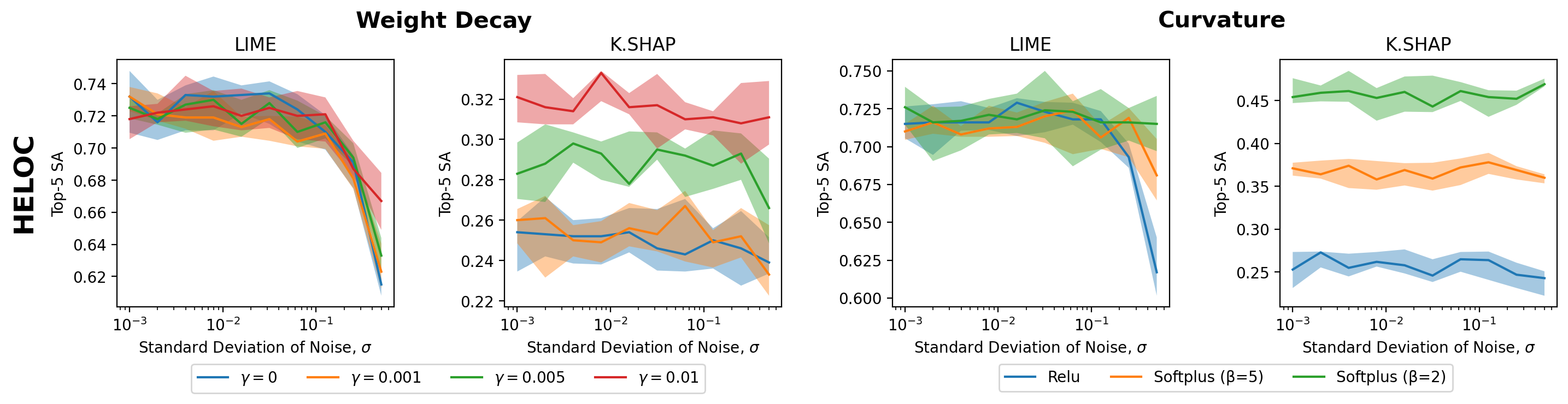}
    \includegraphics[width=0.95\textwidth]{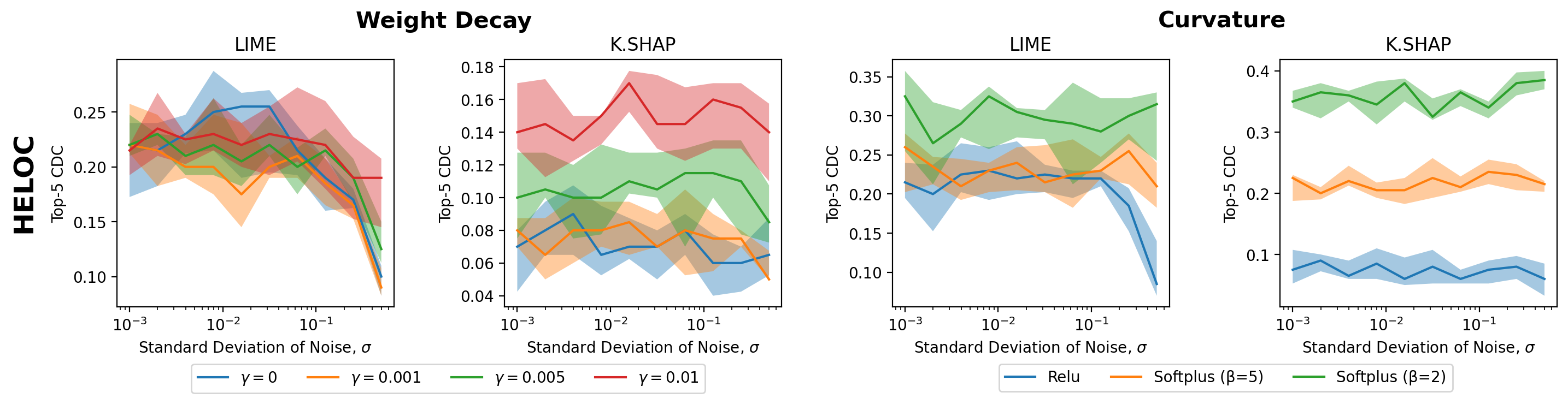}
    \includegraphics[width=0.95\textwidth]{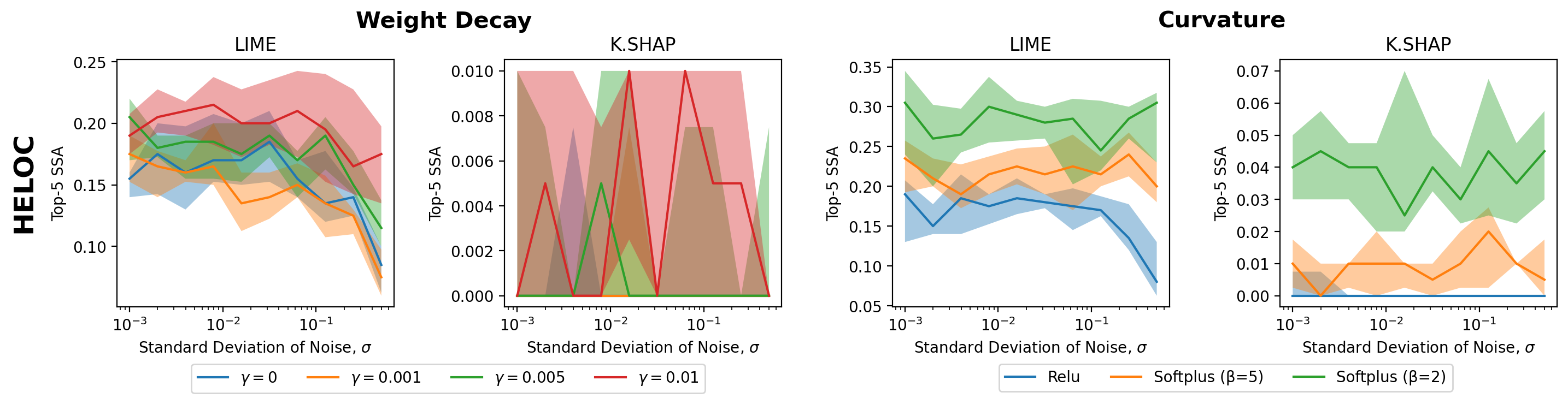}
    \caption{ \small Top-5 consistency. From the top, HELOC SA, HELOC CDC, HELOC SSA. Each row shows the effects of weight decay and curvature as data shift grows for LIME and K.SHAP top-5 metrics. Confidence intervals represent the middle 50\% of values.}
    \label{fig:explanations_retraining_full_heloc}
\end{figure*}

\begin{figure*}[t]
    \centering
    \includegraphics[width=0.95\textwidth]{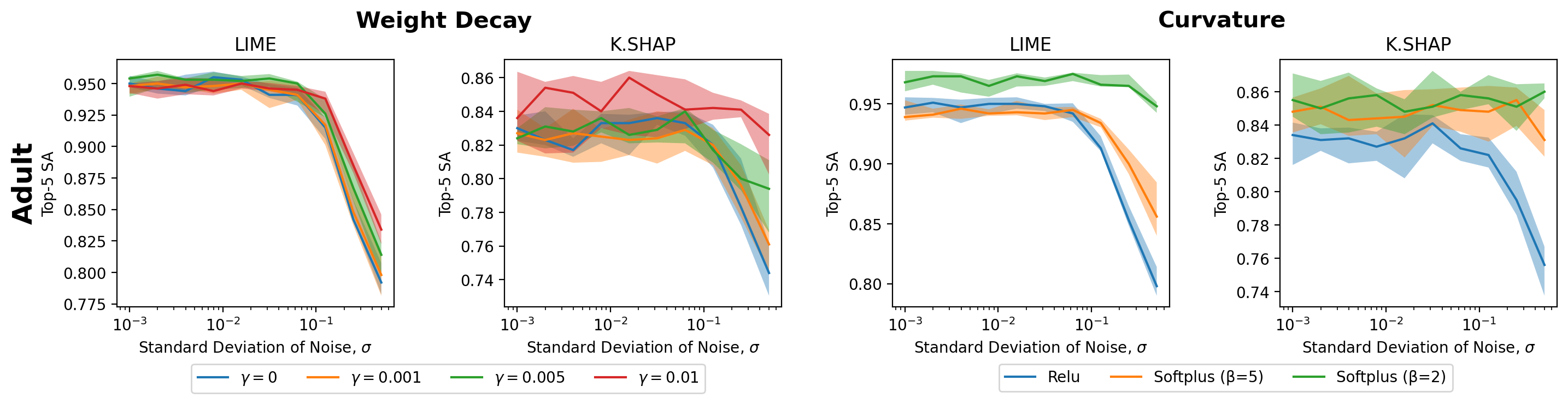}
    \includegraphics[width=0.95\textwidth]{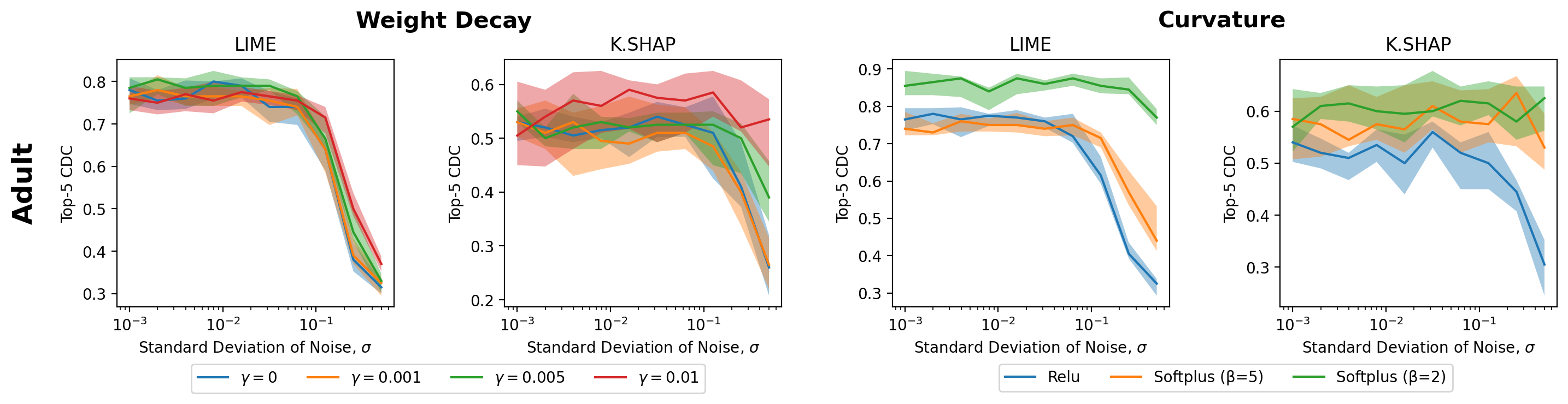}
    \includegraphics[width=0.95\textwidth]{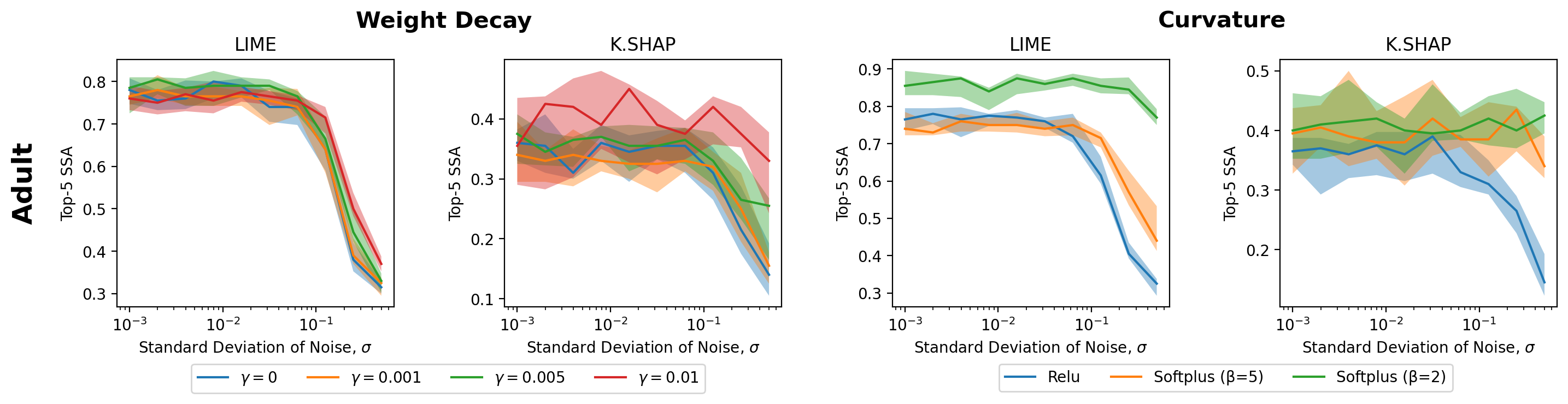}
    \caption{ \small Top-5 consistency. From the top, Adult SA, Adult CDC, Adult SSA. Each row shows the effects of weight decay and curvature as data shift grows for LIME and K.SHAP top-5 metrics. Confidence intervals represent the middle 50\% of values.}
    \label{fig:explanations_retraining_full_adult}
\end{figure*}

\clearpage

\Cref{tab:who_explanations} shows the explanation stability scores for various metrics and explanation techniques for the WHO dataset.  We make the following observations: first, for salience and SHAP (and to a lesser extent, LIME), using a weight decay of 0.01 is much more stable than using smaller weight decays. SmoothGrad is already very stable, so increasing the weight decay has no effect. Second, lowering the model curvature greatly increases the stability of CDC and SSA across explanation techniques, but has less effect on SA. This trend suggests that increasing the model curvature has limited impact on what features are in the top-K (i.e., what SA measures), but has a large impact on keeping the same sign for influential features (even though they might not be in the top K), which is what CDC measures.

\begin{table}[t]
    \small
    \centering
    \caption{\small Explanation Stability for WHO}
\begin{tabular}{ll|rrrrr}
    &            & ReLU ($\gamma=0$)               & ReLU ($\gamma$=0.001)           & ReLU ($\gamma$=0.01)            & SP ($\beta=10$)                 & SP ($\beta=5$)                  \\\toprule
SA  & Salience   & 0.63\textpm 0.01 & 0.64\textpm 0.01 & 0.73\textpm 0.01 & 0.72\textpm 0.02 & 0.77\textpm 0.02 \\
    & SmoothGrad & 0.94\textpm 0.00 & 0.94\textpm 0.00 & 0.95\textpm 0.00 & 0.95\textpm 0.00 & 0.95\textpm 0.00 \\
    & LIME       & 0.59\textpm 0.02 & 0.59\textpm 0.03 & 0.70\textpm 0.03 & 0.65\textpm 0.03 & 0.69\textpm 0.02 \\
    & SHAP       & 0.52\textpm 0.01 & 0.56\textpm 0.02 & 0.67\textpm 0.03 & 0.61\textpm 0.02 & 0.65\textpm 0.02 \\\midrule
CDC & Salience   & 0.80\textpm 0.04 & 0.83\textpm 0.04 & 0.95\textpm 0.01 & 0.94\textpm 0.02 & 0.97\textpm 0.01 \\
    & SmoothGrad & 0.94\textpm 0.00 & 0.94\textpm 0.01 & 0.95\textpm 0.00 & 0.95\textpm 0.00 & 0.96\textpm 0.00 \\
    & LIME       & 0.52\textpm 0.03 & 0.52\textpm 0.05 & 0.57\textpm 0.05 & 0.58\textpm 0.05 & 0.65\textpm 0.04 \\
    & SHAP       & 0.62\textpm 0.02 & 0.70\textpm 0.04 & 0.87\textpm 0.03 & 0.77\textpm 0.03 & 0.86\textpm 0.04 \\\midrule
SSA & Salience   & 0.19\textpm 0.04 & 0.18\textpm 0.03 & 0.30\textpm 0.02 & 0.29\textpm 0.01 & 0.39\textpm 0.03 \\
    & SmoothGrad & 0.91\textpm 0.00 & 0.91\textpm 0.00 & 0.91\textpm 0.00 & 0.91\textpm 0.00 & 0.92\textpm 0.00 \\
    & LIME       & 0.27\textpm 0.02 & 0.25\textpm 0.02 & 0.32\textpm 0.03 & 0.30\textpm 0.03 & 0.35\textpm 0.03 \\
    & SHAP       & 0.19\textpm 0.04 & 0.21\textpm 0.01 & 0.29\textpm 0.04 & 0.23\textpm 0.03 & 0.28\textpm 0.03 \\ \bottomrule
\end{tabular}
\label{tab:who_explanations}
\end{table}

\begin{figure}[t]
    \centering
    \includegraphics[width=\textwidth]{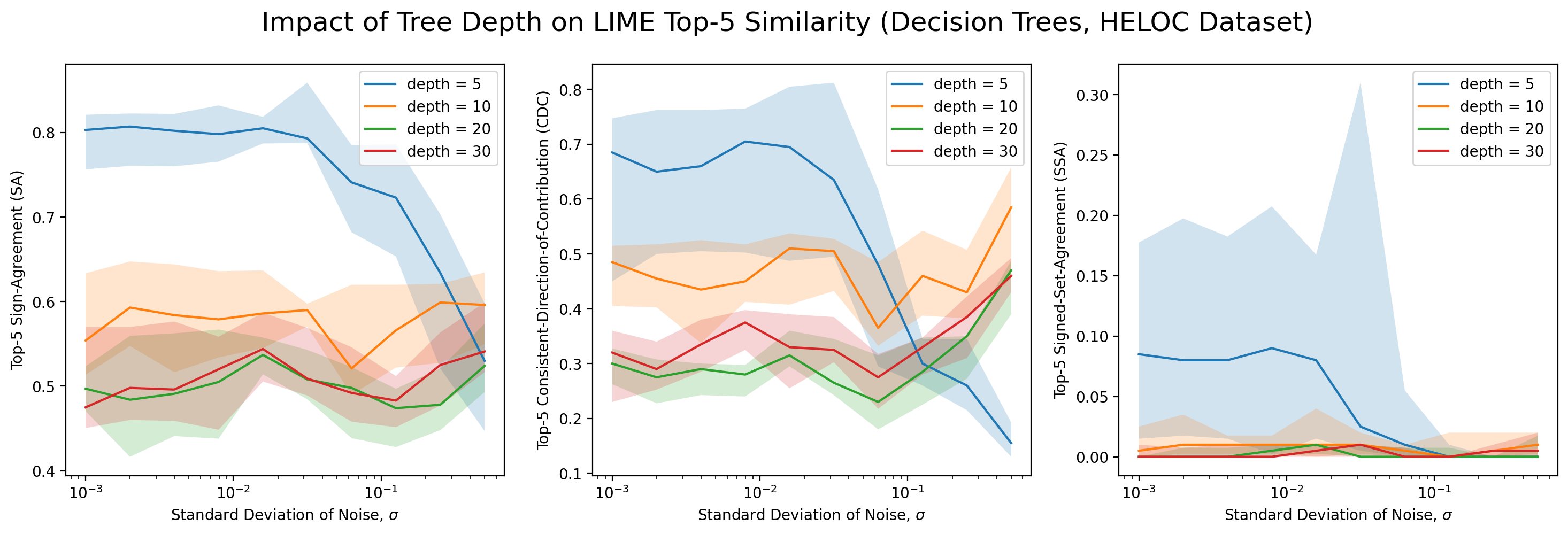}
    \caption{\small Various decision tree depths for the HELOC dataset. Similarity of LIME explanations between base models (original dataset) and retrained models (averaged across 10 random noise seeds for each value of $\sigma$). Left to right show Top-5 SA, CDC, and SSA metrics.}
    \label{fig:dt_heloc_lime}
\end{figure}

\paragraph{Additional experiments with tree-based models} Our theoretical results primarily focus on differentiable models, which encompass popular classes like neural networks, though the trends we have identified may not be universally applicable to all model classes. That said, it is worth noting that prior works analyzing the behavior and stability of ML models and explanation methods also typically rely on linear or differentiable models \cite{dominguez2022causal, pawelczyk2022exploring}. We show that differentiable models with softplus activation functions and higher weight decay values, both of which can be viewed as a way to limit the functional complexity of the model, tend to promote explanation stability. Thus, we posit that explanations of tree-based models could similarly be made more robust by imposing analogous limits on their complexity.

We empirically evaluate explanation stability when retraining decision trees, random forests, and XGBoost models, controlling for tree depth in the first two cases and the $\lambda$ regularization term in XGBoost. Higher $\lambda$ values indicate greater regularization and lower complexity. We evaluated the stability of post-hoc explanations on the HELOC dataset with synthetic noise $\sim\mathcal{N}(0, \sigma^2)$, conducting empirical evaluations using LIME and SHAP. Please refer to \S\ref{sec:eval} for detailed information on the setup of synthetic noise experiments. 
Interestingly, our findings show mixed results. While larger values of $\lambda$ (corresponding to higher regularization constants) or smaller tree depths tend to yield more robust explanations in some cases, the outcomes vary for different explanation methods and dataset shifts.

Figures~\ref{fig:dt_heloc_lime} and ~\ref{fig:dt_heloc_shap} demonstrate the impact of decision tree depth on LIME and SHAP, respectively, for Top-5 SA, CDC, and SSA metrics. Decision trees with lower depth exhibit improved explanation stability for smaller noise values, though the effects can break down as the dataset shifts by larger amounts. While these results support our findings for neural networks to a moderate degree, further research is required to investigate the behaviour that occurs for larger noise values (recall that our theory applies to small dataset shifts).

Figures~\ref{fig:rf_heloc_lime} and ~\ref{fig:rf_heloc_shap} demonstrate the impact of random forest tree depth on LIME and SHAP, respectively, for Top-5 SA, CDC, and SSA metrics. Observe how the explanation stability of random forests is in general much higher than for decision trees-- this would support the idea that the ensemble approach of random forests, which reduces functional diversity and complexity through averaging across multiple decision trees, indeed promotes explanation stability. However, the relative effects of depth are mixed. For LIME, at depth 5, explanations are relatively less stable, though for higher depths, explanations retained strong similarity until sufficiently large synthetic shifts were added ($\sigma>0.1$). For SHAP, we observe the trends that we expect (higher depth reduces stability), though in the case of CDC this was reversed, albeit at very high stability values. Again, further research is required to investigate the behaviour that occurs for larger shifts.

Figures~\ref{fig:xgb_heloc_lime} and ~\ref{fig:xgb_heloc_shap} demonstrate the impact of the XGBoost $\ell_2$ regularization term $\lambda$ on LIME and SHAP, respectively, for Top-5 SA, CDC, and SSA metrics. For LIME, the outcomes appear to be quite volatile, showing a relatively unpredictable range of optimal $\lambda$ values for explanation stability. Interestingly, the added synthetic noise didn't significantly destabilize the explanations, and in some instances, it surprisingly seemed to bolster stability. This counter-intuitive finding underscores the need for further investigation to understand the underlying phenomena driving these trends (though LIME itself may be the significant driving factor). In contrast, for SHAP, the results are much clearer, revealing that higher regularization indeed promotes more stable explanations. Nonetheless, as anticipated, the stability decreases as the dataset undergoes more significant shifts. Notably, the performance of SHAP appeared to be superior to LIME.

\raggedbottom
\interlinepenalty=10000
We want to emphasize that these results are preliminary - more work is needed to generalize these results to larger experiments, to extend these results to real-world (not synthetic) data shifts, and to develop theory for when explanations are be more stable for tree-based models. Notably, we did not perform experiments with the WHO dataset (our real-world shift example) because the results of performing a single experiment were too noisy (with synthetic shifts, the random initialization of the noise can be used to average over multiple trials).

\begin{figure}[t]
    \centering
    \includegraphics[width=\textwidth]{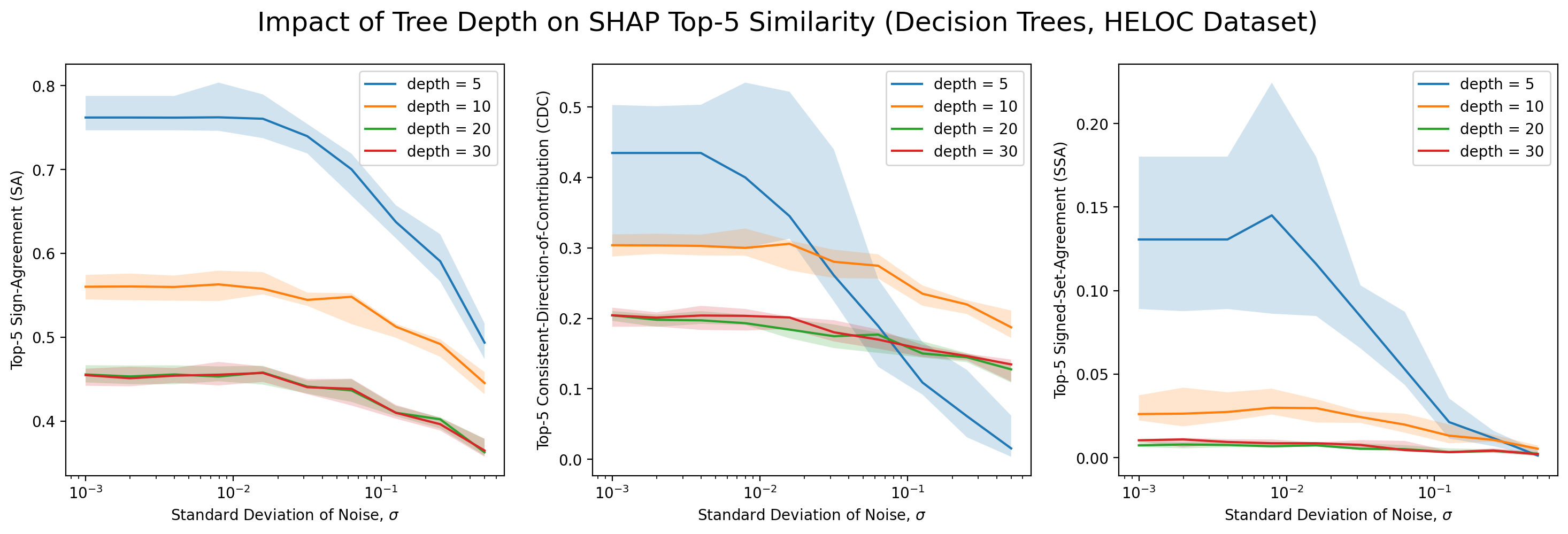}
    \caption{\small Various decision tree depths for the HELOC dataset. Similarity of SHAP explanations between base models (original dataset) and retrained models (averaged across 10 random noise seeds for each value of $\sigma$). Left to right show Top-5 SA, CDC, and SSA metrics.}
    \label{fig:dt_heloc_shap}
\end{figure}

\begin{figure}[t]
    \centering
    \includegraphics[width=\textwidth]{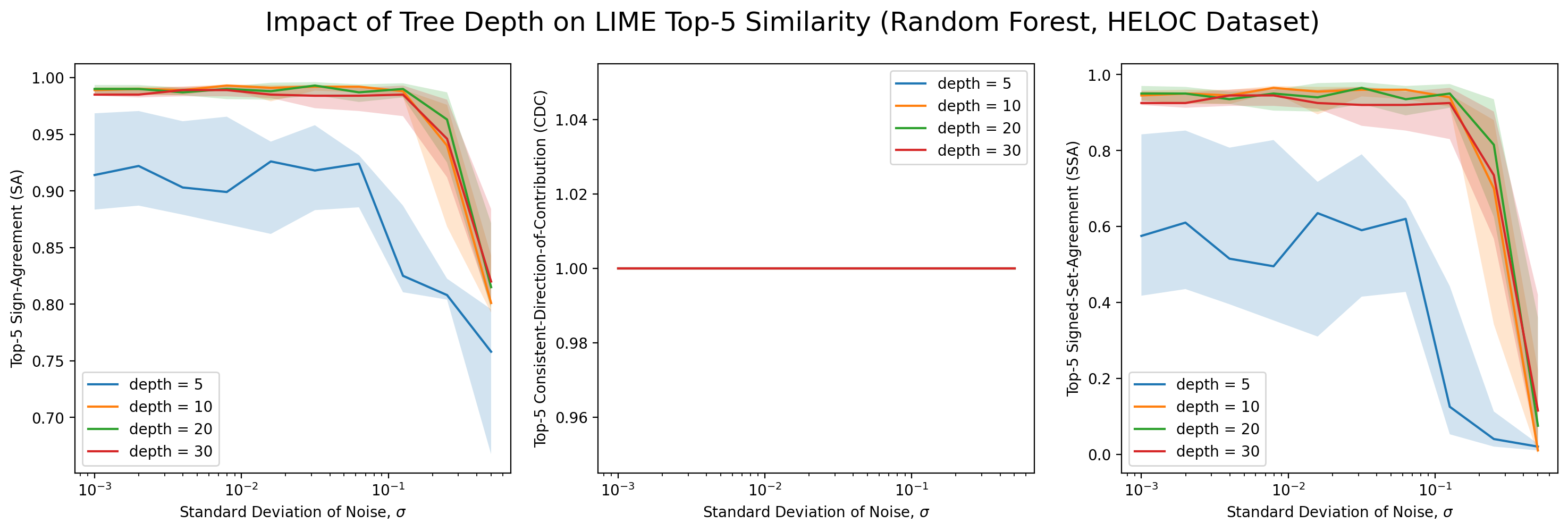}
    \caption{\small Various random forest depths for the HELOC dataset. Similarity of LIME explanations between base models (original dataset) and retrained models (averaged across 10 random noise seeds for each value of $\sigma$). Left to right show Top-5 SA, CDC, and SSA metrics.}
    \label{fig:rf_heloc_lime}
\end{figure}

\begin{figure}[t]
    \centering
    \includegraphics[width=\textwidth]{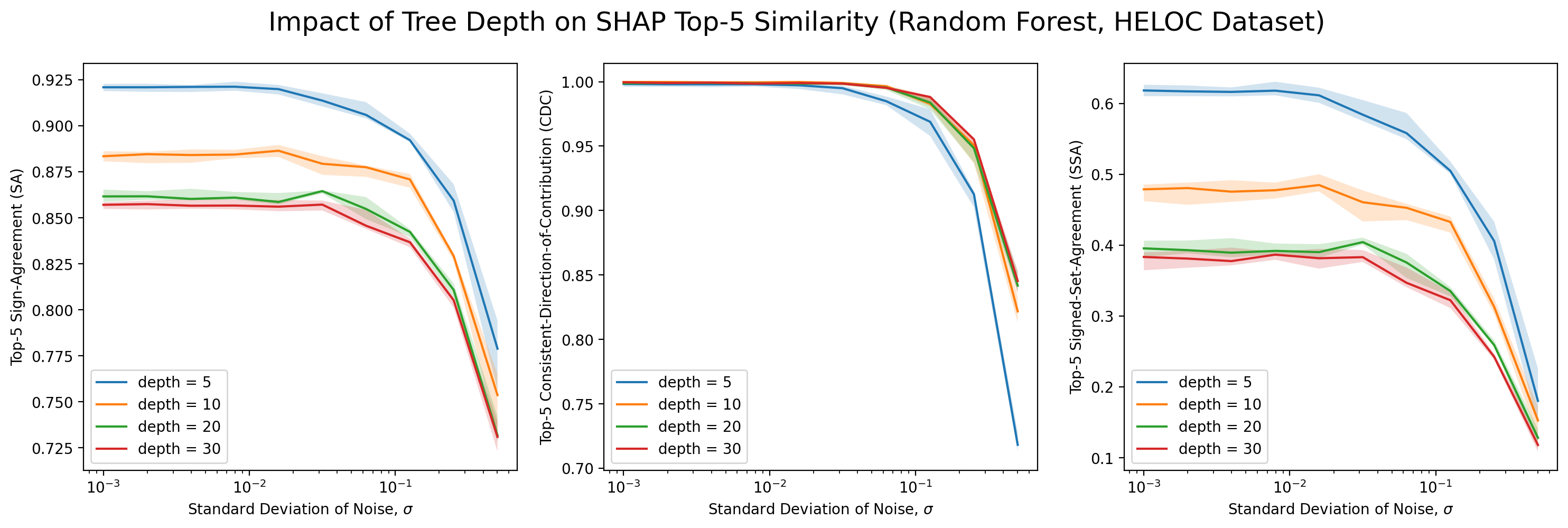}
    \caption{\small Various random forest depths for the HELOC dataset. Similarity of SHAP explanations between base models (original dataset) and retrained models (averaged across 10 random noise seeds for each value of $\sigma$). Left to right show Top-5 SA, CDC, and SSA metrics.}
    \label{fig:rf_heloc_shap}
\end{figure}

\begin{figure}[t]
    \centering
    \includegraphics[width=\textwidth]{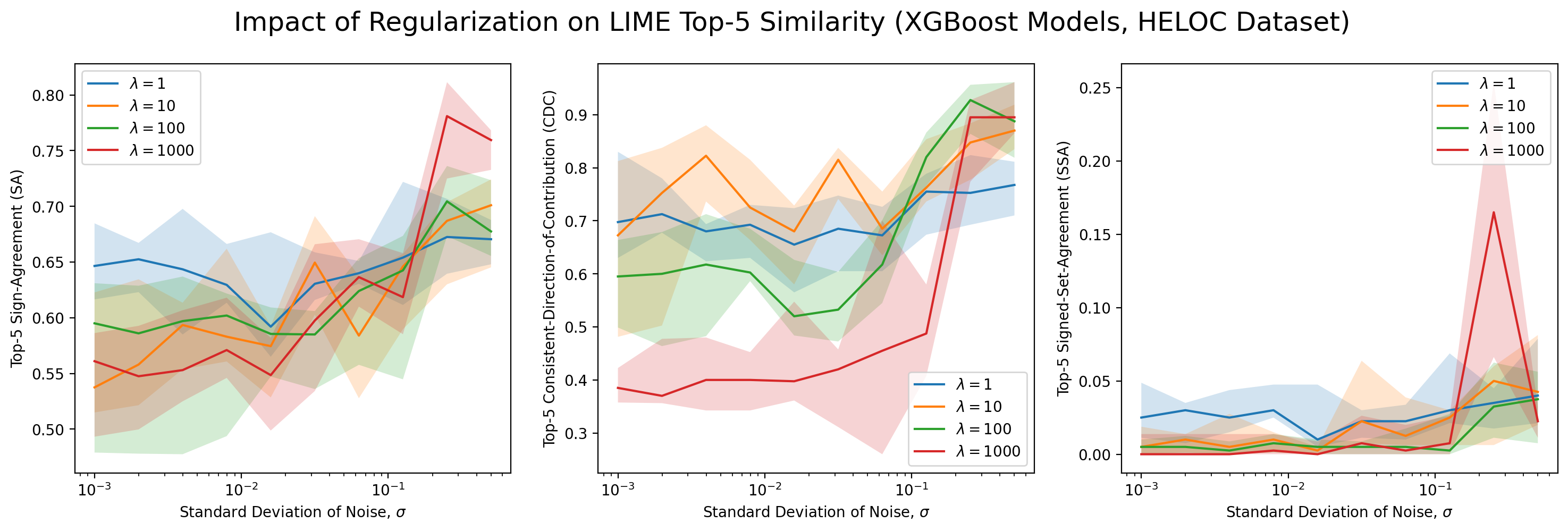}
    \caption{\small Various XGBoost $\lambda$ values for the HELOC dataset. Similarity of LIME explanations between base models (original dataset) and retrained models (averaged across 10 random noise seeds for each value of $\sigma$). Left to right show Top-5 SA, CDC, and SSA metrics.}
    \label{fig:xgb_heloc_lime}
\end{figure}

\begin{figure}[t]
    \centering
    \includegraphics[width=\textwidth]{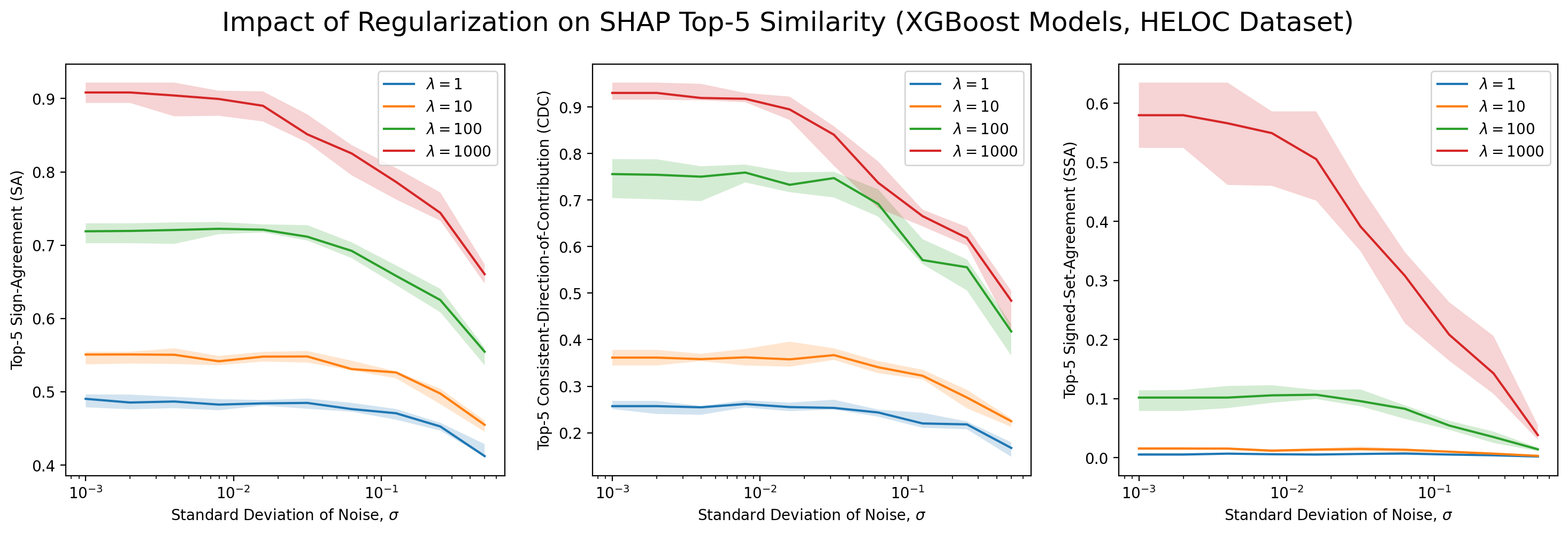}
    \caption{\small Various XGBoost $\lambda$ values for the HELOC dataset. Similarity of SHAP explanations between base models (original dataset) and retrained models (averaged across 10 random noise seeds for each value of $\sigma$). Left to right show Top-5 SA, CDC, and SSA metrics.}
    \label{fig:xgb_heloc_shap}
\end{figure}

\end{document}